 \documentclass[smallextended]{svjour3}   
\usepackage{graphicx}    
\usepackage{amsmath}
\usepackage[sort, numbers]{natbib}
\usepackage{url}
 \usepackage{setspace}
 \usepackage{booktabs}
\usepackage{siunitx}
 \usepackage{color,caption,subcaption}
\usepackage{mathtools}
 \mathtoolsset{showonlyrefs}
 \usepackage{algorithmicx,algpseudocode}
\usepackage{amsmath,amsfonts,amssymb}
 \newcommand{\argmin}{\operatornamewithlimits{argmin}}

\def\makeheadbox{\relax}

\DeclareMathOperator*{\argmax}{arg\,max}

 \newcounter{algno} 
 \usepackage{resizegather}
 \usepackage{multirow}
\begin{document}

\title{Distributed Planning for Serving Cooperative Tasks with Time Windows: A Game Theoretic Approach 
}

\titlerunning{Distributed Planning for Serving Cooperative Tasks with Time Windows}        

\author{Yasin Yaz{\i}c{\i}o\u{g}lu  \and \\
       Raghavendra Bhat \and \\ Derya Aksaray 
}


\institute{Yasin Yaz{\i}c{\i}o\u{g}lu  and Raghavendra Bhat are with Department of Electrical and Computer Engineering, University of Minnesota,
Minneapolis, MN, 55455, U.S.A. \at
              \email{ayasin@umn.edu, bhatx075@umn.edu}           
           \and
           Derya Aksaray is with the Department of Aerospace Engineering and Mechanics, University of Minnesota,
Minneapolis, MN, 55455, U.S.A.  \at
             \email{daksaray@umn.edu}  
}
\date{Received: date / Accepted: date}

\maketitle

\begin{abstract}
We study distributed planning for multi-robot systems to provide optimal service to cooperative tasks that are distributed over space and time. Each task requires service by sufficiently many robots at the specified location within the specified time window. Tasks arrive over episodes and the robots try to maximize the total value of service in each episode by planning their own trajectories based on the specifications of incoming tasks. Robots are required to start and end each episode at their assigned stations in the environment. We present a game theoretic solution to this problem by mapping it to a game, where the action of each robot is its trajectory in an episode, and using a suitable learning algorithm to obtain optimal joint plans in a distributed manner. We present a systematic way to design minimal action sets (subsets of feasible trajectories) for robots based on the specifications of incoming tasks to facilitate fast learning. We then provide the performance guarantees for the cases where all the robots follow a best response or noisy best response algorithm to iteratively plan their trajectories. While the best response algorithm leads to a Nash equilibrium, the noisy best response algorithm leads to  globally optimal joint plans with high probability. We show that the proposed game can in general have arbitrarily poor Nash equilibria, which makes the noisy best response algorithm preferable unless the task specifications are known to have some special structure. We also describe a family of special cases where all the equilibria are guaranteed to have bounded suboptimality. Simulations and experimental results are provided to demonstrate the proposed approach.

\end{abstract}

\section{Introduction}
Multi-robot systems have proven to be effective in various applications such as precision agriculture, environmental monitoring, surveillance, search and rescue, manufacturing, and warehouse automation (e.g., \cite{gonzalez2017fleets, Seyedi19,li2019multi,kapoutsis2017darp, gombolay2018fast, claes2017decentralised}). In many of these applications, the robots need to serve some cooperative tasks that arrive at certain locations during specific time windows. One major requirement for achieving the optimal team performance in such scenarios is to have properly coordinated plans (trajectories) so that sufficiently many robots are present at the right locations and times. 


Multi-robot planning is usually achieved via approximate or distributed algorithms since exact centralized solutions become intractable due to the exponential growth of the joint planning space (e.g., \cite{peasgood2008complete,yu2016optimal}). One of the standard multi-robot planning problems is to reach the goal regions while minimizing the travel time or the distance traveled subject to the constrains such as collision avoidance (e.g., \cite{guo2002distributed, bennewitz2001optimizing}). There are also studies on distributed planning problems where the robots should visit some specified locations before reaching their goal regions (e.g., \cite{Bhattacharya10,Thakur13}). Furthermore, some studies have focused on the planning of robot trajectories to satisfy complex specifications represented as temporal logics (e.g., \cite{ulusoy2013optimality,kress2009temporal,Peterson20, buyukkocak2021planning}). There is also a rich literature on related problems such as vehicle routing, scheduling, and assignment. In vehicle routing problems, a typical objective is to ensure that a given set of locations are visited during the specified time windows subject to various constraints on the vehicles (e.g. \cite{braysy2005vehicle,cordeau2001unified,bullo2011dynamic,Arsie09,aksaray2016dynamic}). On the other hand, scheduling and assignment problems are mainly concerned with the optimal processing of tasks by a number of servers (e.g., \cite{andersson2006multiprocessor,dai2019task,michael2008distributed,gombolay2018fast,Arslan07,nunes2017taxonomy,wang2020coupled}). These problems are typically NP-hard and solved via approximation algorithms. 
In this paper, we investigate a \emph{distributed  task execution} (DTE) problem, where a homogeneous team of mobile robots need to plan their trajectories in a distributed manner to optimally serve the cooperative tasks that arrive over episodes. Each task requires service at the corresponding location by sufficiently many robots within the specified time window. Robots aim to maximize the value of completed tasks by each of them planning its own trajectory, which must start and end at the assigned station in the environment. One of the main differences of this problem from standard planning and vehicle routing problems is that the tasks may demand more complex coordination in robot trajectories rather than being instantaneously completed when their locations are visited by one of the robots. For example, tasks may require multiple robots for multiple time steps, may be preemptable or non-preemptable, may involve complex preconditions (e.g., first a certain number of robots should together move a heavy object, then a single robot or multiple robots can complete the task). Accordingly, the proposed problem formulation can accommodate a wide range of cooperative tasks with time windows.

This paper proposes a game theoretic solution to the DTE problem by designing a corresponding game and utilizing game theoretic learning to drive the robots to joint plans that maximize the global objective function, i.e., the total value of service provided to the tasks. Similar game theoretic formulations were presented in the literature to achieve coordination in problems such as vehicle-target assignment (e.g., \cite{Arslan07}), coverage optimization (e.g., \cite{Yazicioglu13NECSYS,Yazicioglu17TCNS,Zhu13}), and dynamic vehicle routing (e.g., \cite{Arsie09}). In our proposed method, we map the DTE problem to a game where the action of each robot is defined as its trajectory (plan) in an episode. We show that some feasible trajectories can never contribute to the global objective in this setting, irrespective of the trajectories of other robots. For any given set of tasks, by excluding such inferior trajectories, we obtain a game with a minimal action space that not only contains globally optimal joint plans but also facilitates fast learning. We then provide the performance guarantees for the cases where all the robots follow a best response or noisy best response algorithm to iteratively plan their trajectories in a distributed manner. While the best response algorithm leads to a Nash equilibrium, the noisy best response algorithm leads to globally optimal joint plans with high probability when the noise is small. We then show that the resulting game can in general have arbitrarily poor Nash equilibria, which makes the noisy best response algorithm preferable unless the task specifications have some special structure. We also describe a family of special cases where all the Nash equilibria are guaranteed to be near-optimal and the best response algorithm may be used to monotonically improve the joint plan and reach a near-optimal solution. Finally, we present simulations and experiments to demonstrate the proposed approach. 

This paper is a significant extension of our preliminary work in \cite{Bhat19} with the following main differences: 1) We extend the problem formulation to accommodate a broader range of tasks compared to \cite{Bhat19}, which only considered tasks that can be completed in one time step when there are sufficiently many robots. In this modified setting, the tasks are also allowed to change over time and their specifications are available to the robots before each episode. 2) We facilitate faster learning by designing significantly smaller action spaces based on the specifications of incoming tasks. We use different learning algorithms, define  the information needed by each agent to follow these algorithms, and provide a price of anarchy analysis.
3) We provide new theoretical results, numerical simulations, and experiments on a team of drones.

The organization of this paper is as follows: Section \ref{prob} presents the DTE problem. Section \ref{prelim} provides some game theory preliminaries. Section \ref{design} presents the game design. Section \ref{learn} is on the learning dynamics and performance guarantees. Simulation results are presented in Section \ref{sims}. Section \ref{exp-res} presents the experiments on a team of quadrotors. Finally, Section \ref{conc} concludes the paper.

\section{Problem Formulation}
\label{prob}
This section presents the distributed task execution (DTE) problem, where a homogeneous team of $n$ mobile robots, $R=\{r_1, r_2, \hdots, r_n\}$, need to plan their trajectories in each episode to optimally serve the incoming cooperative tasks with time windows. 
\subsection{Notation}
We use $\mathbb{Z}$ (or $\mathbb{Z}_+$) to denote the set of integers (or positive integers) and $\mathbb{R}$ (or $\mathbb{R}_+$) to denote the set of real (or positive real) numbers. For any pair of vectors $x,y \in \mathbb{R}^n$, we use $x \leq y$ (or $x<y$) to denote the element-wise inequalities, i.e., $x_i \leq y_i$ (or $x_i < y_i$) for all ${i=1,2, \hdots, n}$.

\subsection{Formulation}

We consider a discretized environment represented as a 2D grid, 
\begin{equation}
P=\{1,2, \hdots , \bar{x}\} \times \{1,2, \hdots , \bar{y}\},
\end{equation}
where $\bar{x},\bar{y} \in \mathbb{Z}_+$ denote the number of cells along the corresponding directions. In this environment, some cells may be occupied by static obstacles and the robots are free to move over the feasible cells ${P_F \subseteq P}$. Each cell in the grid represents a sufficiently large space that can accommodate any number of robots at the same time. There are stations located at a subset of the feasible cells ${S \subseteq  P_F}$. Robots recharge and get ready for the next episode at their assigned stations. Each robot is assigned to a specific station (multiple robots can be assigned to the same station), where its trajectory must start and end in each episode. 

Each episode consists of $T$ time steps and the trajectory of each robot $r_i \in R$  over an episode is denoted as ${\mathbf{p}_i=\{p_i^0, p_i^1, \hdots, p_i^T \}}$. The robots can move to any of the feasible neighboring cells within one time step. Accordingly, when a robot is at some cell $p=(x,y)\in P_F$, at the next time step it has to be within $p$'s neighborhood on the grid $N(p)\subseteq P_F$, which is given as
\begin{equation}
    \label{motion}
    N(p)= \{(x',y') \in P_F \mid |x'-x| \leq 1, |y'-y| \leq 1\}.
\end{equation}
 For any $r_i\in R$, the set of feasible trajectories, $\mathbf{P}_i$, is 
\begin{equation}
    \label{traj-set}
    \mathbf{P}_i= \{\mathbf{p}_i \mid p_i^0=p_i^T=\sigma_i, \; p_i^{t+1} \in N(p_i^t), \;  \forall t<T \},
\end{equation}
where $\sigma_i \in S$ is the location that contains the station of robot $r_i$. { As per \eqref{traj-set}, a trajectory $\mathbf{p}_i$ is feasible if it satisfies two conditions: 1) it starts and ends at the assigned station $\sigma_i$, and 2) each position along the trajectory is in the neighborhood of preceding position.} The Cartesian product of the sets of feasible trajectories is denoted as $\mathbf{P}=\mathbf{P}_1 \times \hdots \times \mathbf{P}_n$.
A sample environment with some obstacles and three stations is illustrated in Fig. \ref{env}.

\begin{figure}[htb]
\centering
\includegraphics[trim =0mm 0mm 0mm 0mm, clip,scale=0.6]{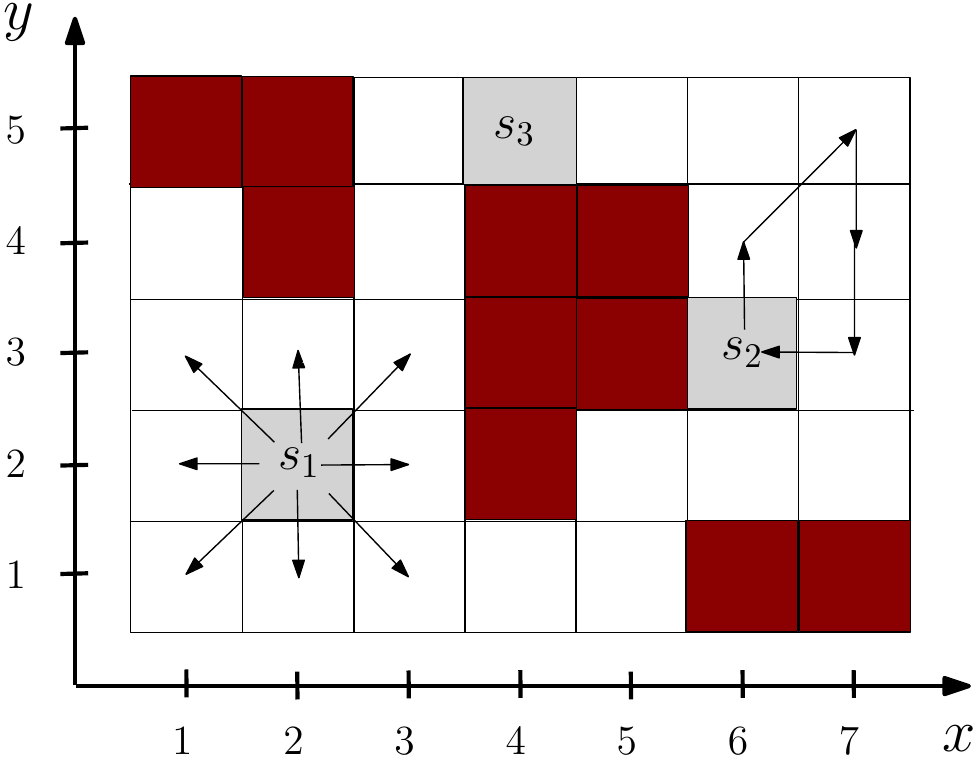}
\caption{A sample environment with obstacles (red) and three stations (gray). Robots can move to any neighboring cell in one time step. The neighboring cells of $s_1$ are shown on the left. A feasible trajectory of length five that starts and ends at $s_2$ is shown on the right. }
\label{env}
\end{figure}

A new set of tasks $\tau=\{\tau_1, \tau_2, \hdots, \tau_m\}$ is received in each episode. Each task is defined as a tuple, $\tau_i= \{\ell_i, t^a_i,t^d_i,v_i\}$,
where ${\ell_i \in P_F}$ is the location, ${t^a_i<t^d_i \in \{0,\hdots,T\}}$ are the arrival and departure times (time window),  and the value function $v_i$ is a mapping from the numbers of robots serving the task during the time window $\{t^a_i,\hdots,t^d_i-1\}$ to the resulting value. In order to serve a task, a robot should spend at least one time step at that location, i.e., $p_i(t)=p_i(t+1)=\ell_i$, during the time window. We refer to each repetition of position, $p_i(t)=p_i(t+1)$, along a trajectory as a \emph{stay}. In this setting, multiple tasks can arrive at the same location in an episode, but we assume\footnote{{ This assumption is only made to simplify the notation and presentation in our derivations. In Section \ref{mtaskext}, we discuss how this assumption can be lifted to use our proposed approach in cases where multiple tasks are simultaneously active at the same location.}} that the time windows of such tasks with identical locations do not overlap, i.e., 
\begin{equation}
\label{nooverlap}
\ell_i = \ell_j \Rightarrow \min \{t^d_i,t^d_j\} \leq \max\{t^a_i, t^a_j\}, \forall i \neq j.
\end{equation}
Accordingly, the number of robots staying at each location uniquely determines their impact on performance since they all serve the same task (if there is one).  For any  ${\mathbf{p} \in \mathbf{P}}$ and $\tau_i \in \tau$, we use $c_i(\mathbf{p},t) \in \mathbb{Z}$ to denote the number of robots that stay at the location of the task, $\ell_i$, from time $t$ to $t+1$, i.e., 
\begin{equation}
\label{rtq}
c_i(\mathbf{p},t) =\big|\{r_j \in R \mid p_j^t = p_j^{t+1} = \ell_i\} \big|.
\end{equation}
 We use $\mathbf{c}_i(\mathbf{p})$ to denote the vector of counter values during the time window of the task $\tau_i$, i.e.,
\begin{equation}
\label{rtq2}
\mathbf{c}_i(\mathbf{p})=[c_i(\mathbf{p},t^a_i), \hdots, c_i(\mathbf{p},t^d_i-1)]^\text{T}, 
\end{equation}
and the resulting value from the task is $v_i(\mathbf{c}_i(\mathbf{p})) \in [0,\bar{v}_i]$, where $\bar{v}_i \in \mathbb{R}_+$ is the maximum value that can be obtained from the task (e.g., when the task is completed as desired). To accommodate various types of cooperative tasks, we do not make any assumptions on the value functions $v_i$ except for the mild assumption that having more robots can never hurt the outcome, i.e., 
\begin{equation}
\label{moreok}
\mathbf{c}_i(\mathbf{p}) \geq \mathbf{c}_i(\mathbf{p}') \Rightarrow v_i(\mathbf{c}_i(\mathbf{p})) \geq v_i(\mathbf{c}_i(\mathbf{p}')).
\end{equation}
In the remainder of the paper, we will say that task $\tau_i$ is completed if it yields the maximum value $\bar{v}_i$. While we will use tasks with binary value functions (0 or $\bar{v}_i$) in our examples for simplicity, our methods are applicable to tasks with more generic value functions with higher resolution. Under this model, the robots are assumed to be capable of achieving the required low-level coordination for each task (e.g., moving an object together). Accordingly, each task is completed if it is served by sufficiently many robots during the corresponding time window. We quantify the performance in an episode via the total value from tasks, i.e.,
\begin{equation}
\label{obj}
f(\mathbf{p})= \sum_{\tau_i \in \tau}v_i(\mathbf{c}_i(\mathbf{p})),
\end{equation}
where $\mathbf{p} \in \mathbf{P}$ denotes the trajectories of all robots. 
In each episode, the robots plan their own trajectories to maximize \eqref{obj}. Such a distributed coordination problem can be solved by utilizing methods from machine learning, optimization, or game theory (e.g., \cite{Bu08,Boyd11,Marden09}). In this paper, we will study this problem from a game theoretic perspective. 

\section{Game Theory Preliminaries}
\label{prelim}
A finite \emph{strategic game} $\Gamma= (I, A, U)$ is defined by three elements: (1) the set of \emph{agents (players)} $I=\{1, 2, \hdots, n\}$, (2) the \emph{action space} $A= A_1 \times A_2 \times \hdots\ \times A_n$, where each $A_i$ is the \emph{action set} of agent $i$, and (3) the set of \emph{utility functions} $U= \{U_1, U_2, \hdots , U_n\}$, where each $U_i : A \mapsto \mathbb{R}$ is a mapping from the action space to the set of real numbers. Any \emph{action profile} $a \in A$ can be represented as ${a=(a_i,a_{-i})}$, where $a_i \in A_i$ is the action of agent $i$ and $a_{-i}$ denotes the actions of all other agents. 
An action profile $a^* \in A$ is a \emph{Nash equilibrium} if no agent can increase its own utility by unilaterally changing its action, i.e.,
\begin{equation}
\label{NE}
 U_i(a_i^*, a_{-i}^*) = \max_{a_i \in A_i} U_i(a_i,a_{-i}^*), \; \forall i \in I.
\end{equation}
 A game is called a \emph{potential game} if there exists a function, $\phi : A \mapsto \mathbb{R}$, such that for each player $i \in I$, for every $a_i, a'_i \in A_i$, and for all $a_{-i} \in A_{-i}$,  
\begin{equation}
\label{PG}
 U_i(a'_i, a_{-i})- U_i(a_i, a_{-i}) = \phi(a'_i, a_{-i})- \phi(a_i, a_{-i}).
\end{equation}

{ Accordingly, whenever an agent unilaterally changes its action in a potential game, the resulting change in its own utility equals the resulting change in $\phi$, which is called the \emph{potential function} of the game. }

In game theoretic learning, the agents start with arbitrary initial actions and follow a learning algorithm to update their actions in a repetitive play of the game. At each round $k \in \{0, 1, 2, \hdots \}$, each agent $i \in  I$ plays an action $a_i(k)$ and receives the utility $U_i(a(k))$.  In general, a learning algorithm maps the observations of an agent from the previous rounds $\{0, 1, \hdots, k-1\}$ to its action in round $k$.  In this paper, we will only consider algorithms with a single-stage memory, where the action in round $k$ depends only on the observation in round $k-1$, to achieve the desired performance.
For potential games, one such learning algorithm that achieves almost sure convergence to a Nash equilibrium is the best-response (BR) (e.g., see \cite{Young04} and the references therein). While any global maximizer of the potential function is necessarily a Nash equilibrium, potential games may also have suboptimal Nash equilibria. For any potential game $\Gamma$ with the set of Nash equilibria $A^* \subseteq A$,  the comparison of the worst and the best Nash equilibria can be achieved through the measure known as the $\emph{price of anarchy}$ (PoA), i.e., $
\text{PoA}(\Gamma)= \max\limits_{a\in A^*}\phi(a)/\min\limits_{a\in A^*}\phi(a)$. For potential games with high PoA, noisy best-response algorithms such as \emph{log-linear learning} (LLL)  \cite{Blume93} can be used to have the agents spend most of their time at the global maximizers of $\phi(a)$. More specifically, LLL induces an irreducible and aperiodic Markov chain over the action space such that the limiting distribution, $\mu_\epsilon$, satisfies
\begin{equation}
\label{sstab}
\lim_ {\epsilon \rightarrow 0^+} \mu_{\epsilon}(a) >0 \Longleftrightarrow \phi(a) \geq \phi(a'), \forall a' \in A.
\end{equation}
{ Based on \eqref{sstab}, as the noise parameter of LLL, $\epsilon$, goes down to zero (as LLL becomes similar to BR), only the action profiles that globally maximize $\phi$ maintain a non-zero probability in the resulting limiting distribution $\mu_\epsilon$. }
\section{Game Design}
\label{design}
In this section, we map the DTE problem to a potential game, $\Gamma_{\text{DTE}}$, whose potential function is equal to \eqref{obj}. Once such a game is designed by defining the action sets and the utility functions, learning algorithms such as BR or LLL can be used to reach the desired joint plans in a distributed manner. 
\subsection{Action Space Design} 
The impact of each agent on the overall objective in \eqref{obj} is determined only by its trajectory. Accordingly, one possible way to design the action space is to define each action set as $A_i= \mathbf{P}_i$, which contains all the feasible trajectories. However, the number of feasible trajectories, $|\mathbf{P}_i|$, grows exponentially with the episode length $T$, and the learning process typically gets slower as the agents need to explore a larger number of possibilities. Furthermore, many standard learning algorithms such as BR and LLL require the updating agent in each round to compute all the possible utilities that can be obtained by switching to any of its feasible actions. Hence, a large action set increases not only the number of rounds needed for the convergence of learning but also the computation time required by the updating agents in each round. Motivated by the practical importance of computation times, we aim to design the smallest action sets that can still yield the optimal joint plan.

We design the minimal action sets by excluding a large number of feasible trajectories that can never be essential to the overall performance, regardless of the trajectories taken by the other robots. For example, if a trajectory $\mathbf{q}_i \in \mathbf{P}_i$ does not serve any task, i.e., there is no $\tau_j \in \tau$ such that $q_i^{t}=q_i^{t+1}=\ell_j$ for some $t \in \{t^a_j, \hdots, t^d_j-1\}$, then it is guaranteed that robot $i$ will not be contributing to the global score in \eqref{obj} when traversing $\mathbf{q}_i$ since it will not be contribution to any counter in \eqref{rtq}. Accordingly,  such $\mathbf{q}_i$ can be removed from the action set without causing any performance loss. Furthermore, if a trajectory $\mathbf{p}_i$ has all the task-serving stays contained in some other trajectory $\mathbf{q}_i$, then removing $\mathbf{q}_i$ from the action set (while keeping $\mathbf{p}_i$) would not degrade the overall performance. Accordingly, we define each action set as follows: 
\begin{align}\label{actset}
    A_i = \argmin_{\emptyset \subset A_i' \subseteq \mathbf{P}_i} & \quad |A_i'|\\
    s.t. & \quad  \eqref{actionsetc2}, \nonumber
\end{align}
where the constraint is
\begin{equation}
\label{actionsetc2}
\forall \mathbf{q}_i \in \mathbf{P}_i \setminus A_i', \exists \mathbf{p}_i \in A_i': p_i^{t}=p_i^{t+1}=q_i^{t}, \forall t \in t^*(\mathbf{q}_i,\tau),
\end{equation}
and $t^*(\mathbf{q}_i, \tau)$ is the set of times where $\mathbf{q}_i$ involves a stay at a task location within the corresponding time window, i.e.,  
\begin{equation}
\label{actwin}
t^*(\mathbf{q}_i, \tau)= \{t \mid \exists \tau_j \in \tau, \; q_i^{t}=q_i^{t+1}=\ell_j, \; t^a_j \leq t <t^d_j\}.
\end{equation}
Accordingly, each robot's action set $A_i$ is the smallest non-empty subset of its all feasible trajectories $\mathbf{P}_i$ such that for every excluded trajectory ${\mathbf{q}_i \in \mathbf{P}_i \setminus A_i}$, there exists a trajectory $\mathbf{p}_i \in A_i$ such that any stay in $\mathbf{q}_i$ within the corresponding active time window is also included in $\mathbf{p}_i$, i.e., \eqref{actionsetc2}. Note that any $\mathbf{q}_i$ with no task-serving stays, i.e., $t^*(\mathbf{q}_i,\tau)=\emptyset$, is trivially removed from the action set as \eqref{actionsetc2} does not impose any restriction on the removal of such $\mathbf{q}_i$. Our next result formally shows that this reduced action space does not cause any suboptimality.

\begin{lemma}
\label{maxpreserved}
For the sets of feasible trajectories $\mathbf{P}_i$ as in \eqref{traj-set} and the action sets $A_i$ as in \eqref{actset}, $\mathbf{P}=\mathbf{P}_1 \times \hdots \times \mathbf{P}_n$ and $A=A_1 \times \hdots \times A_n$ satisfy
\begin{equation}
\label{lem11}
\max_{\mathbf{p} \in \mathbf{P}}f(\mathbf{p}) = \max_{\mathbf{p} \in A} f(\mathbf{p}).
\end{equation}
\end{lemma}
\begin{proof}
Since $A \subseteq \mathbf{P}$, 
\begin{equation}
\label{lem11a}
\max_{\mathbf{p} \in \mathbf{P}}f(\mathbf{p}) \geq \max_{\mathbf{p} \in A} f(\mathbf{p}).
\end{equation}
Now, let $\mathbf{q} \in \mathbf{P} \setminus A$ be a maximizer of $f(\mathbf{p})$, i.e.,
\begin{equation}
\label{lem12}
     f(\mathbf{q}) = \max_{\mathbf{p} \in \mathbf{P}} f(\mathbf{p}).
\end{equation}
Due \eqref{actionsetc2}, there exist $\mathbf{p} \in A$ such that, for every robot $i$, all the stays in $\mathbf{q}_i$ that takes place at a task location  within the corresponding time windows are also included in $\mathbf{p}_i$. To be more specific, for any  $t \in\{0, \hdots, T-1\}$ and $\tau_j \in \tau$, we have
\begin{equation}
\label{lem13}
   q_i^t= q_i^{t+1}=\ell_j, \; t^a_j \leq t <t^d_j \Rightarrow p_i^t= p_i^{t+1}=q_i^t.
\end{equation}
Accordingly, any stay in $\mathbf{q}_i$ that may contribute to a counter (see \eqref{rtq} and \eqref{rtq2}) is also included in $\mathbf{p}_i$, which implies $\mathbf{c}_j(\mathbf{p}) \geq \mathbf{c}_j(\mathbf{q})$ for every task $\tau_j \in \tau$. Hence, due to \eqref{moreok} and \eqref{obj}, we have $f(\mathbf{p}) \geq f(\mathbf{q})$ and, due to \eqref{lem12},
\begin{equation}
\label{lem14}
\max_{\mathbf{p} \in A}f(\mathbf{p}) \geq \max_{\mathbf{p} \in \mathbf{P}} f(\mathbf{p}).
\end{equation}
Consequently, \eqref{lem11a} and \eqref{lem14} together imply \eqref{lem11}.
\qed \end{proof}


\subsection{Utility Design} We utilize the notion of \emph{wonderful life utility} \cite{Tumer04} to design a game whose potential function is the total value in \eqref{obj}. Accordingly, we define the utility of each robot as its marginal contribution to the total value, i.e.,
\begin{equation}
\label{util}
U_i (\mathbf{p}) =  \sum_{\tau_j \in \tau}\left [ v_j(\mathbf{c}_j(\mathbf{p})) - v_j(\mathbf{c}_j(\mathbf{p}_{-i})) \right ],
\end{equation}
where $\mathbf{c}_j(\mathbf{p}_{-i})$ is the counter associated with $\tau_j$ that disregards agent $i$, i.e.,
\begin{equation}
\label{rtqmarg}
c_j(\mathbf{p}_{-i},t) =|\{r_k \in R \setminus \{r_i\} \mid p_k^t = p_k^{t+1} = \ell_i\}|.
\end{equation}

{ As per \eqref{util}, the utility of each robot $r_i$, i.e., $U_i(\mathbf{p})$, is equal to the total value of tasks that are completed  under the trajectories $\mathbf{p}$ and would not be completed without $r_i$ (under the trajectories $\mathbf{p}_{-i}$). }
\begin{lemma}
\label{potgam} 
Utilities in \eqref{util} lead to a potential game $\Gamma_{\emph{DTE}}= (R, A, U)$ whose potential function equals the total value received from the tasks, i.e., ${\phi(\mathbf{p})= f(\mathbf{p})}$.
\end{lemma}
\begin{proof}
Let $\mathbf{p}_i \neq \mathbf{p}_i' \in A_i$ be two possible trajectories for any robot $i$, and let $\mathbf{p}_{-i}$ denote the trajectories of all other other robots. Using \eqref{util}, we have 
 \begin{equation}
\label{potgam3}
U_i(\mathbf{p}_i,\mathbf{p}_{-i})-U_i(\mathbf{p}_i',\mathbf{p}_{-i}) =  \sum_{\tau_j \in \tau}v_j(\mathbf{p}_i,\mathbf{p}_{-i}) - \sum_{\tau_j \in \tau} v_j(\mathbf{p}_i',\mathbf{p}_{-i}),
\end{equation}
which, together with \eqref{obj}, implies
\begin{equation}
\label{potgam4}
U_i(\mathbf{p}_i,\mathbf{p}_{-i})-U_i(\mathbf{p}_i',\mathbf{p}_{-i})= f(\mathbf{p}_i,\mathbf{p}_{-i})-f(\mathbf{p}_i',\mathbf{p}_{-i}).
\end{equation}
Consequently, $f(\mathbf{p})$ is the potential function for $\Gamma_{\emph{DTE}}$.
\end{proof}

 \emph{\textbf{Example 1}}: Consider the environment in Fig.~\ref{env} with 3 robots, two stationed at $s_1$ in cell $(2,2)$ and one stationed at $s_3$ in cell $(4,5)$. Let $T=6$, and consider a single task at location $(3,3)$ that can be completed at any time during the episode by moving some boxes as illustrated in Fig. \ref{ex-boxs}. More specifically, the task first requires moving a heavy box, which can be handled by at least 2 robots, and then moving the two light boxes, each of which can be handled by a single robot, to the heavy box's initial location. Suppose that moving the heavy box to its desired position takes one time step. Similarly, a single robot can move one light box to the desired location in one time step. 
 
 \begin{figure}[htb!]
\centering
\includegraphics[trim =0mm 0mm 0mm 0mm, clip,scale=0.6]{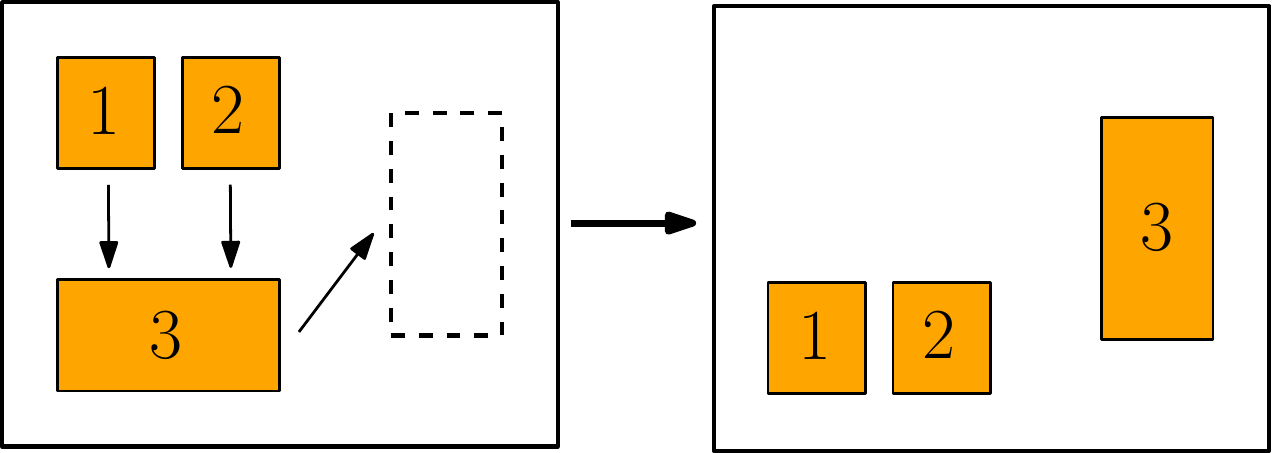}
\caption{An example task where the goal is to move there boxes from the configuration on the left into the configuration on the right. The task can be completed by first moving Box 3 to its desired position (dashed) and then moving Boxes 1 and 2 to Box 3's initial position. }
\label{ex-boxs}
\end{figure}
 
 Such a task first requires  at least 2 robots to serve this location together for one time step, and then the total number of robots serving that location within the remaining time to be at least 2.  Suppose that the task yields a value of 1 if successfully completed. Accordingly, this task can be represented with the following specifications: $\ell_1=(3,3)$, $t^a_1=0$, $t^d_1=6$,
 and
\begin{equation}
    \label{vex}
    v_1(\mathbf{c}_1(\mathbf{p}))=\left\{\begin{array}{ll} 1, \mbox{ if $\exists i, \; [\mathbf{c}_1(\mathbf{p})]_i  \geq 2$, $\sum \limits_{j>i}[\mathbf{c}_1(\mathbf{p})]_j\geq 2$ } \\ 0,\mbox{ otherwise.}\end{array} \right.,
\end{equation}
where $[\mathbf{c}_1(\mathbf{p})]_i$ and $[\mathbf{c}_1(\mathbf{p})]_j$  denote the $i^{th}$ and $j^{th}$ entries of the counter vector $\mathbf{c}_1(\mathbf{p})$. The value function in \eqref{vex} implies that the task is completed if there exists an index $i$ such that 1) the $i^{th}$ entry of $\mathbf{c}_1(\mathbf{p})$ is at least two, and 2) the summation of the entries with indices $j>i$ are at least two. Given these task specifications, let the trajectories of the three robots over the episode of six time steps, i.e., for $t=0,1, \hdots, 6$, be as follows:  
$$\mathbf{p}_1=\{(2,2), (3,3), (3,3), (3,3), (3,3), (3,3), (2,2)\},$$ $$\mathbf{p}_2=\{(2,2), (3,3), (3,3), (3,3), (3,3), (3,3), (2,2)\},$$ $$\mathbf{p}_3=\{(4,5), (3,4), (3,3), (3,3), (3,3), (3,4), (4,5)\}.$$ In that case, the task is completed since $\mathbf{c}_1(\mathbf{p})= [0, 2, 3, 3, 2, 0]^{\text{T}}$ as per \eqref{rtq2}. The task can be completed without $r_1$ or $r_2$ since
$${\mathbf{c}_1(\mathbf{p}_{-1})=\mathbf{c}_1(\mathbf{p}_{-2})= [0, 1, 2, 2, 1, 0]^{\text{T}}}.$$ 
Accordingly, $r_1$ and $r_2$ receive the utilities $U_1(\mathbf{p})=U_2(\mathbf{p})=0$. Similarly, if $r_3$ is removed from the system, $r_1$ and $r_2$ can still complete the task since ${\mathbf{c}_1(\mathbf{p}_{-3})= [0, 2, 2, 2,2,0]^{\text{T}}}$. Hence, $r_3$ also receives a utility of zero, $U_3(\mathbf{p})=0$. As such, although the task is completed in this example, none of the robots would receive a utility since their marginal contributions to the value received are all equal to zero, i.e., the task would still be completed by the remaining two robots if any single robot was removed from the system. 
 
\subsection{Communication and Information Requirements}
Both the action set in \eqref{actset} and the utility in \eqref{util} can be computed by each robot based on local information. To be more specific, for each robot $r_i$ we first define two sets: 1) set of reachable tasks (robot can serve the location of the task for at least one time step and return to its station within $T$ steps), 
\begin{equation}
\label{loc_task}
\tau^i_{local}=\{\tau_j \in \tau \mid dist(\ell_j,\sigma_i) < T/2\},
\end{equation}
where $dist(\ell_j,\sigma_i)$ denotes the minimum number of transitions (allowing diagonal transitions) needed to reach the task location $\ell_j$ from the station $\sigma_i$, and 2) set of robots $r_j$ who has a common reachable task with $r_i$, i.e.,
\begin{equation}
\label{loc_rob}
R^i_{local}=\{r_j \in R \mid \tau^i_{local} \cap \tau^j_{local} \neq \emptyset\}.
\end{equation}

Accordingly, if each robot $r_i$ has the following information:  1) the specifications of each task $\tau_j \in \tau^i_{local}$, and 2) the trajectory $\mathbf{p}_j$ of each robot $r_j \in R^i_{local}$, then each robot $r_i$ can compute its own utility in \eqref{util}. Furthermore, for any trajectory $\mathbf{q}_i$  $t^*(\mathbf{q}_i,\tau)=t^*(\mathbf{q}_i,\tau^i_{local})$ since any task $\tau_j \notin \tau^i_{local}$ can never be served under any feasible trajectory of $r_i$. Hence, such local information is also sufficient for each robot to compute its action set as per \eqref{actset}. We assume that each robot $r_i$ is able to obtain the specifications of tasks in $\tau^i_{local}$ and the trajectories of robots in $R^i_{local}$ through local communications.

{In Section \ref{learn}, we present a learning process where the robots repetitively play $\Gamma_{DTE}$ and, at each round, a randomly picked robot updates its trajectory based on the utilities it can obtain from different trajectories. Accordingly, first, all the robots are given the specifications of their reachable tasks and they broadcast their initial trajectories to their neighbors in the beginning of the learning process. Then, at each round, only the updating agent needs to broadcast its new trajectory. In such a learning process over $K$ rounds, each robot $r_i$ would need to communicate, by either broadcasting its updated trajectory or receiving an update from another robot in $R^i_{local}$, in approximately $K|R^i_{local}|/n$ of the rounds, which defines the approximate communication load of the learning process on each robot. }
{
\subsection{Tasks with Identical Locations and Overlapping Time Windows}
\label{mtaskext}
Our derivations so far were based on the assumption that tasks arriving at the same location do not have overlapping time windows. This assumption was made just to simplify the notation and define the action of each robot as its trajectory. In this subsection, we show how this assumption can be easily lifted to use our proposed approach when multiple tasks may be active at the same location. In particular, this extension is achieved with minor modifications to the action sets $A_i$ in \eqref{actset} and the counters $\mathbf{c}_i$ in \eqref{rtq2}. We denote these modified versions as $A_i^+$ and $\mathbf{c}_i^+$, which are defined below.

Once each $A_i$ is generated according to \eqref{actset},  $A_i^+$ can be obtained from $A_i$ by adding a second dimension to the actions. This second dimension is used to distinguish between the cases where different tasks are served under the same trajectory. Given a set of $m$ tasks $\tau$,   each $\mathbf{p}_i^+=(\mathbf{p}_i,\mathbf{z}_i) \in A_i^+$ consists of a trajectory $\mathbf{p}_i \in A_i$ and an additional sequence $\mathbf{z}_i=\{z_i^0, z_i^1, \hdots, z_i^{T}\} \in \{0,\tau_1, \hdots, \tau_m\}^{T+1}$ that indicates which task is being served by $r_i$ at time $t$ (e.g., $z_i^t=0$ if no task is served by $r_i$ at time $t$). More specifically, let $\theta (a_i,\tau,t)$ denote the tasks in $\tau$ that can be served at time $t$ by a robot following the trajectory $a_i$, i.e., the set of tasks $\tau_j$ such that $r_i$ stays at the location of the task $\ell_j$ at time $t$ and the task is active at time $t$ ($t^a_j \leq t < t^d_j$):
\begin{equation}
    \label{thetaset}
    \theta (\mathbf{p}_i,\tau,t)=  \{\tau_j \in \tau \mid p_i^t=p_i^{t+1}=\ell_j, t^a_j \leq t <t^d_j \}.
\end{equation}
Accordingly, the action set $A_i^+$ is defined as
\begin{equation}
\label{aplus}
A_i^+=\{(\mathbf{p}_i,\mathbf{z}_i) \mid \mathbf{p}_i \in A_i, z_i^t = 0  \mbox{ if } \theta (\mathbf{p}_i,\tau,t) = \emptyset, z_i^t \in \theta (\mathbf{p}_i,\tau,t) \mbox{ if } \theta(\mathbf{p}_i,\tau,t) \neq \emptyset \},
\end{equation}
where each action consists of a trajectory $\mathbf{p}_i$ and a sequence $\mathbf{z}_i$ that explicitly states the tasks $r_i$ plans to serve while taking the trajectory $\mathbf{p}_i$. 
As such, $|A_i^+|\geq |A_i|$ and $A_i^+$ is obtained by minimally increasing the size of the action set so that each action uniquely identifies the service provided by each robot. Furthermore, $|A_i|=|A_i^+|$ when $\tau$ contains no tasks with identical locations and overlapping time windows. 

\emph{\textbf{Example 2}}: Consider the environment in Fig.~\ref{env} with a single robot $r_1$ stationed at $s_1$ in cell $(2,2)$. Let $T=4$, and consider two tasks $\tau=\{\tau_1,\tau_2\}$ such that the locations are $\ell_1=\ell_2=(3,3)$ and the arrival and departure times are $t^a_1=0$, $t^d_1=3$, $t^a_2=2$, $t^d_2=4$.  In this example, \eqref{actset} results in $A_1$ consisting of a single trajectory: $\mathbf{p}_1=\{(2,2), (3,3), (3,3), (3,3), (2,2)\}$. Along this trajectory, it is clear that the robot is serving $\tau_1$ during its stay at $t=1$ since $\tau_2$ has not arrived yet. However, this trajectory does not uniquely describe which task is served during the stay at $t=2$ as both tasks are active at that time. By growing the action set as in \eqref{aplus}, we obtain an action set containing two actions:
$
A_1^+=\{(\mathbf{p}_1,\{0,\tau_1, \tau_1, 0\}),(\mathbf{p}_1,\{0,\tau_1, \tau_2, 0\}) \}.
$
Note that any choice from $A_1^+$ uniquely determines the service provided by the robot.

In addition to extending the action sets as $A_i^+$, we also need to make a minor modification to the definition of the counters associated with the tasks. More specifically, we replace \eqref{rtq} and \eqref{rtq2} with 
\begin{equation}
\label{rtqplus}
c_i^+(\mathbf{p}^+,t) =\big|\{r_j \in R \mid p_j^t = p_j^{t+1} = \ell_i, z_j^t=\tau_i \} \big|,
\end{equation}
\begin{equation}
\label{rtq2plus}
\mathbf{c}_i^+(\mathbf{p}^+)=[c_i^+(\mathbf{p}^+,t^a_i), \hdots, c_i^+(\mathbf{p}^+,t^d_i-1)]^\text{T},
\end{equation}
where $\mathbf{p}^+ = [\mathbf{p}^+_1, \hdots, \mathbf{p}^+_n] \in A_1^+ \times \hdots \times A_n^+$ is the action profile in the modified action space. Accordingly, each robot $r_j$ contributes to the counter of a task $\tau_i$ if it stays at the corresponding location during the corresponding time window and commits to serving $\tau_i$ as per $\mathbf{p}^+_i=(\mathbf{p}_i,\mathbf{z}_i)$. By using these modifications to the action sets and the counters and defining the utilities accordingly as per \eqref{util}, i.e.,
\begin{equation}
\label{utilplus}
U_i (\mathbf{p}^+) =  \sum_{\tau_j \in \tau}\left [ v_j(\mathbf{c}_j^+(\mathbf{p}^+)) - v_j(\mathbf{c}_j^+(\mathbf{p}^+_{-i})) \right ],
\end{equation}
we complete the design of the $\Gamma_{DTE}$ for the cases where multiple tasks may be simultaneously active at the same location. In the remainder of the paper, we will continue discussing our derivations in the setting where the action of each robot is defined as its trajectory (tasks with identical locations do not have overlapping time windows). However, all of our results can be easily extended to the generalized case by using the modified game design presented here.

}
\section{Learning Dynamics}
\label{learn}
Once the specifications of tasks in the upcoming episode are provided to the robots and each robot computes its action set as in \eqref{actset}, various learning algorithms can be used by the robots in a repetitive play of $\Gamma_{\text{DTE}}$ to optimize \eqref{obj} in a distributed manner. We consider two conventional learning algorithms with different performance guarantees for potential games: Best-Response (BR), which ensures convergence to a Nash equilibrium, and Log-Linear Learning (LLL), which ensures the stochastic stability of joint plans that maximize the potential function (e.g., see \cite{Blume93,Young04} and the references therein). In this setting, the learning algorithm serves as a distributed optimization protocol where each robot $r_i$ updates its intended plan based on the specifications of the tasks in $\tau^i_{local}$, which is defined in \eqref{loc_task}, and the plans of the other robots in $R^i_{local}$, which is defined in \eqref{loc_rob}.  Under these algorithms, a random agent is selected to make a unilateral update in each round, and that agent plays a best response or a noisy best response (log-linear) to the recent actions of the other agents. The selection of a random agent at each round can be achieved in a distributed manner without a global coordination, for instance by using the asynchronous time model in \cite{Boyd06}. In the best-response algorithm, the updating agent picks a maximizer of its utility function (assuming the actions of others will stay the same)  as its next actions (picks the current action if it is already a maximizer). In the log-linear learning, the updating agent randomizes the next action over the whole action set with probabilities determined by the corresponding utilities (similar to the softmax function). Accordingly, the agent assigns much higher probabilities to the actions that would yield higher utility. Both algorithms are formally described below.

{\small
\begin{center}
\begin{tabular}{l l}
\rule[0.08cm]{8.5cm}{0.03cm}\\
\textbf{Best Response (BR)}\\
\rule[0.08cm]{8.5cm}{0.02cm}\\
\mbox{\small $\;1:\;$}\textbf{initialization:} $k=0$, arbitrary $\mathbf{p}(0)\in A$.\\
 \mbox{\small $\;2:\;$}\textbf{repeat} \\
\mbox{\small $\;3:\;$}\hspace{0.1cm} Pick a random agent $r_i \in R$.\\
\mbox{\small $\;4:\;$}\hspace{0.1cm} Compute $\text{BR}(\mathbf{p}_{-i}(k))= \argmax \limits_{\mathbf{p}_i \in A_i} U_i(\mathbf{p}_i, \mathbf{p}_{-i}(k))$. \\
\mbox{\small $\;5:\;$}\hspace{0.1cm} $\mathbf{p}_i(k+1)=\left\{\begin{array}{ll} \mbox{$\mathbf{p}_i(k)$, if $\mathbf{p}_i(k) \in \text{BR}(\mathbf{p}_{-i}(k))$,} \\ \mbox{Random in $\text{BR}(\mathbf{p}_{-i}(k))$, otherwise.}\end{array}\right.$
\\
\mbox{\small $\;6:\;$}\hspace{0.1cm} $\mathbf{p}_{-i}(k+1)=\mathbf{p}_{-i}(k)$.
\\
\mbox{\small $\;7:\;$}\hspace{0.1cm} $k=k+1.$\\
\mbox{\small $\;8:\;$} \textbf{end repeat}\\

\rule[0.08cm]{8.5cm}{0.02cm}\\
\end{tabular}
\end{center}
}

{\small
\begin{center}
\begin{tabular}{l l}
\rule[0.08cm]{8.5cm}{0.03cm}\\
\textbf{Log-Linear Learning (LLL)}\\
\rule[0.08cm]{8.5cm}{0.02cm}\\
\mbox{\small $\;1:\;$}\textbf{initialization:} $k=0$, arbitrary $\mathbf{p}(0)\in A$, small $\epsilon >0$.\\
 \mbox{\small $\;2:\;$}\textbf{repeat} \\
\mbox{\small $\;3:\;$}\hspace{0.1cm} Pick a random agent $r_i \in R$.\\
\mbox{\small $\;4:\;$}\hspace{0.1cm} Randomize the next action of $r_i$:  \\
\hspace{0.65cm} $\Pr[\mathbf{p}_i(k+1)=\mathbf{p}_i] \sim \exp{\left(\dfrac{U_i(\mathbf{p}_i, \mathbf{p}_{-i}(k))}{\epsilon}\right)}$, $\forall  \mathbf{p}_i \in A_i$. \\
\mbox{\small $\;5:\;$}\hspace{0.1cm} $\mathbf{p}_{-i}(k+1)=\mathbf{p}_{-i}(k)$.\\
\mbox{\small $\;6:\;$}\hspace{0.1cm} $k=k+1.$\\
\mbox{\small $\;7:\;$} \textbf{end repeat}\\

\rule[0.08cm]{8.5cm}{0.02cm}\\
\end{tabular}
\end{center}
}

 Our next results provide the formal guarantees on the evolution of the global score in \eqref{obj} when robots follow BR or LLL in a repeated play of $\Gamma_{\textnormal{DTE}}$.  { In particular, we first show that if all robots follow BR, then the value of completed tasks converges to a value within $1/PoA(\Gamma_{\textnormal{DTE}})$ of the maximum possible value with probability one as the number of rounds, $k$, goes to infinity.}
\begin{theorem}
\label{Solved}
Let $\Gamma_{\textnormal{DTE}}=(R,A,U)$ be designed as per \eqref{actset} and \eqref{util}. If all robots follow BR in a repeated play of $\Gamma_{\textnormal{DTE}}$, then with probability one
\begin{equation}
\label{thmeq-br}
\lim_{k \to \infty} f(\mathbf{p}(k)) \geq \dfrac{\max\limits_{\mathbf{q}\in \mathbf{P}} f(\mathbf{q})}{ \textnormal{PoA}(\Gamma_{\textnormal{DTE}}) }.
\end{equation}
\end{theorem}
\begin{proof}
Since $\Gamma_{\text{DTE}}=(R,A,U)$ is a potential game, the best response dynamics achieve convergence to a Nash equilibrium with probability one (e.g., see \cite{Young04} and the references therein). From the definition of PoA, this implies that with probability one
\begin{equation}
\label{thmeq-br1}
\lim_{k \to \infty} f(\mathbf{p}(k)) \geq \dfrac{\max\limits_{\mathbf{q}\in A} f(\mathbf{q})}{ \text{PoA}(\Gamma_{\text{DTE}}) }.
\end{equation}
Using \eqref{thmeq-br1} together with \eqref{lem11}, we obtain \eqref{thmeq-br}.
\end{proof}

{ Our next result shows that if all robots follow LLL with an arbitrarily small noise parameter $\epsilon>0$, then the probability of obtaining trajectories that maximize the total value becomes arbitrarily close to one as the number of rounds, $k$, goes to infinity. }
\begin{theorem}
\label{Solved2}
Let $\Gamma_{\textnormal{DTE}}=(R,A,U)$ be designed as per \eqref{actset} and \eqref{util}. If all robots follow log-linear learning (LLL) in a repeated play of $\Gamma_{\textnormal{DTE}}$, then 
\begin{equation}
\label{thmeq}
\lim_{\epsilon \to 0^+}\lim_{k \to \infty} \Pr \left [f(\mathbf{p}(k)) = \max_{\mathbf{q}\in \mathbf{P}} f(\mathbf{q}) \right]=1.
\end{equation}

\end{theorem}
\begin{proof}
Since $\Gamma_{\text{DTE}}=(R,A,U)$ is a potential game with the potential function $f(\mathbf{p})$, LLL induces an irreducible and aperiodic Markov chain with the limiting distribution $\mu_\epsilon$ over $A$ (e.g., see \cite{Young04} and the references therein) such that { as $\epsilon$ (the noise parameter of LLL) goes down to zero, only the global maximizers of $f(\mathbf{p})$ maintain a non-zero probability in $\mu_\epsilon$, i.e.,}
\begin{equation}
\label{th1}
\lim_ {\epsilon \rightarrow 0^+} \mu_{\epsilon}(\mathbf{q}) >0 \Longleftrightarrow f(\mathbf{q}) = \max_{\mathbf{q}' \in A} f(\mathbf{q}').
\end{equation}
Accordingly, the trajectories at the $k^{th}$ round of learning, $\mathbf{p}(k)$, satisfy 
\begin{equation}
\label{th2}
\lim_{\epsilon \to 0^+}\lim_{k \to \infty} \Pr \left [\mathbf{p}(k) =\mathbf{q}\right]>0 \Longleftrightarrow f(\mathbf{q}) = \max_{\mathbf{q}' \in A} f(\mathbf{q}'),
\end{equation}
which implies
\begin{equation}\label{th3}
\lim_{\epsilon \to 0^+}\lim_{k \to \infty} \Pr \left [f(\mathbf{p}(k)) = \max_{\mathbf{q}\in A} f(\mathbf{q}) \right]=1.
\end{equation}
Using \eqref{th3} together with \eqref{lem11}, we obtain \eqref{thmeq}.

\end{proof}

{ Based on Theorems \ref{Solved} and \ref{Solved2}, both BR and LLL provide guarantees on the trajectories $\mathbf{p}(k)$ as the number of rounds, $k$, goes to infinity. In practice, there would be a finite amount of time for planning the trajectories before each episode in the DTE problem. Accordingly, our proposed solution is to have the robots update their trajectories via learning in $\Gamma_{\textnormal{DTE}}$ over a finite number of rounds (available time between episodes) and then dispatch according to the resulting trajectories. When the learning horizon is sufficiently long, the performance induced by BR or LLL would be close to the respective limiting behavior. More specifically, for sufficiently large $k$: 1) $f(\mathbf{p}(k))$ is within $1/PoA(\Gamma_{\textnormal{DTE}})$ of the maximum possible value with a high probability under BR,  and 2) $f(\mathbf{p}(k))$ equals the maximum possible value with a high probability when LLL is executed with a sufficiently small noise parameter $\epsilon$.}

In light of Theorems \ref{Solved} and \ref{Solved2}, a major consideration in choosing the learning algorithm is the price of anarchy (PoA). If all Nash equilibria yield reasonably good $f(\mathbf{p})$, then best-response type algorithms can achieve the desired performance. Such an approach has the benefit of having a monotonic increase in the global objective in \eqref{obj} as the robots update their plans, i.e., $f(\mathbf{p}(k+1)) \geq f(\mathbf{p}(k))$ for all $k \geq 0$. On the other hand, if some Nash equilibria are highly suboptimal, noisy best-response type algorithms such as LLL can be used to ensure that the learning process does not converge to an undesirable Nash equilibrium and, while $f(\mathbf{p}(k))$ does not increase monotonically under the resulting learning process, a global optima of $f(\mathbf{p})$ is observed most of the time in the long-run as robots keep updating their plans. We continue our analysis by investigating $\textnormal{PoA}(\Gamma_{\textnormal{DTE}})$.

\subsection{Price of Anarchy}
\label{secpoa}

 We first provide an example to show that $\textnormal{PoA}(\Gamma_{\textnormal{DTE}})$ can be arbitrarily large in general  when there is not a special structure in the task specifications.

 \emph{\textbf{Example 3}}: Consider the environment in Fig \ref{env}, and let each episode consist of three time steps ($T=3$). Let there be two robots $\{r_1, r_2\}$, both stationed at $s_1$ in cell $(2,2)$. Suppose that we have three tasks with identical time windows, $t^a_1=t^a_2=t^a_3=0$ and $t^d_1=t^d_2=t^d_3=4$, and different locations: $\ell_1=(1,1)$,  $\ell_2=(1,2)$, $\ell_3=(1,3)$. Each task requires the handling of some boxes and can be completed if sufficiently many robots (depending on the weight of boxes) stay at that location for one time step, i.e., the value functions have the form 
 \begin{equation}
    \label{vex2}
    v_i(\mathbf{c}_i(\mathbf{p}))=\left\{\begin{array}{ll} \bar{v}_i, \mbox{ if $\max (\mathbf{c}_i(\mathbf{p})) \geq c^*_i$,} \\ 0,\mbox{ otherwise.}\end{array}\right.
\end{equation}
Suppose that $c^*_1=c^*_2=1$ (light boxes), $c^*_3=2$ (heavy boxes), and $\bar{v}_3 \gg \bar{v}_1,\bar{v_2}$. In this setting, using \eqref{actset}, the action set of each robot consists of three trajectories: going to any of the three task locations, staying there for one step and coming back. It can be shown that this scenario has three Nash equilibria with the following outcomes: 1) $r_1$ completes $\tau_1$ and $r_2$ completes $\tau_2$, 2) $r_1$ completes $\tau_2$ and $r_2$ completes $\tau_1$, and 3) $r_1$ and $r_2$ together complete $\tau_3$. While the first two cases result in a total value of $\bar{v}_1+\bar{v}_2$, the last option yields a total value of $\bar{v}_3$. Accordingly, $
\text{PoA}(\Gamma_{\text{DTE}})$ equals $\bar{v}_3/(\bar{v}_1+\bar{v}_2)$,
which can be arbitrarily large.

Example 3 shows that in general $\text{PoA}(\Gamma_{\text{DTE}})$ may be arbitrarily large. However, there are also instances of the problem where  $\text{PoA}(\Gamma_{\text{DTE}})$ is small. We will first give a definition and then present a family of such cases with a bound on $\text{PoA}(\Gamma_{\text{DTE}})$.

\begin{definition}
(Simple Task)  A task ${\tau_i= \{\ell_i, t^a_i,t^d_i,v_i\}}$ is simple if it can be completed by one robot in one time step, i.e., the value function has the form
\begin{equation}
    \label{defeq}
    v_i(\mathbf{c}_i(\mathbf{p}))=\left\{\begin{array}{ll} \bar{v}_i, \mbox{ if $\max (\mathbf{c}_i(\mathbf{p})) \geq 1$,} \\ 0,\mbox{ otherwise.}\end{array}\right.
\end{equation}
\end{definition}

One real-life example of a simple task is an aerial monitoring task that requires taking images from a specific location within a specific time window. When the grid cells correspond to sufficiently small regions, such a monitoring task can be completed by a single drone within a single time step. Similarly, certain pick-up and delivery or manipulation tasks can be completed by a single robot in a single time step.

\begin{theorem}
\label{POA-T}
Let $\Gamma_{\textnormal{DTE}}=(R,A,U)$ be designed as per \eqref{actset} and \eqref{util}. For a system with $n$ robots and $m$ tasks, if there is only one station and all the tasks are simple, then the price of anarchy of $\Gamma_{\textnormal{DTE}}$ is bounded as
\begin{equation}
\label{thmeq-poa1}
\textnormal{PoA}(\Gamma_{\textnormal{DTE}}) \leq \max \left (\dfrac{m}{n}, 1 \right).
\end{equation}
\end{theorem}
\begin{proof}
 Consider a single-station game with $n$ robots and $m$ simple tasks. Note that in such a game, all the robots have identical action sets. Let $A^* \subseteq A$ be the set of Nash equilibria, and let $\mathbf{p}^* \in A^*$ be any Nash equilibrium of the game. We analyze each of the two possible cases separately and show that \eqref{thmeq-poa1} holds in both cases:
 
\emph{ Case 1 - All tasks are completed under $\mathbf{p}^*$}: In this case, clearly $ f(\mathbf{p}^*)=\max\limits_{\mathbf{p} \in A^*}f(\mathbf{p})$ since $f(\mathbf{p})$ cannot exceed the total value of tasks in $\tau$. 

\emph{Case 2 - Some tasks are not completed under $\mathbf{p}^*$}: Let $\tau'\neq  \emptyset$ be the set of incomplete task under $\mathbf{p}^*$. Since all tasks are simple, any robot could switch to a trajectory completing some $\tau_j \in\tau'$ to receive a utility of $\bar{v}_j$. {Accordingly, since $\mathbf{p}^*$ is a Nash equilibrium, each agent's utility must be larger than the value of any incomplete task, i.e.,
\begin{equation}
\label{thmeq-poa2}
U_i(\mathbf{p}^*) \geq \max_{\tau_j \in \tau'}\bar{v}_j, \; \forall i \in \{1, \hdots,n\},
\end{equation}
which implies that each agent must be receiving a positive utility by being essential for the completion of at least one task (the task would be incomplete without that agent) due to \eqref{util}. Since each task is simple, multiple agents cannot be essential for the same completed task. Hence, at least $n$ tasks must be completed. Furthermore, each completed task's value is included in at most one agent's utility. Hence, the total value of completed tasks cannot be less than the total utility of the agents, i.e.,
\begin{equation}
\label{thmeq-poa2a}
\sum_{\tau_j \in \tau\setminus\tau'} \bar{v}_j \geq \sum_{i=1}^n U_i(\mathbf{p}^*).
\end{equation}
Since the number of completed tasks is at least $n$, the total value of incomplete tasks must be upper bounded by $m-n$ times the value of the most valuable incomplete task, i.e.,
\begin{equation}
\label{thmeq-poa2b}
\sum_{\tau_j \in \tau'} \bar{v}_j \leq  (m-n)\max_{\tau_j \in \tau'}\bar{v}_j.
\end{equation}
Using \eqref{thmeq-poa2} and \eqref{thmeq-poa2a}, we obtain
\begin{equation}
\label{thmeq-poa3}
f(\mathbf{p}^*) = \sum_{\tau_j \in \tau\setminus\tau'} \bar{v}_j \geq \sum_{i=1}^n U_i(\mathbf{p}^*) \geq n \max_{\tau_j \in \tau'}\bar{v}_j.
\end{equation}
Since $\max\limits_{\mathbf{p} \in A^*}f(\mathbf{p})$ cannot exceed the total value of tasks in $\tau$,  we have
\begin{equation}
\label{thmeq-poa4}
\max\limits_{\mathbf{p} \in A^*}f(\mathbf{p}) \leq  \sum_{\tau_j \in \tau\setminus\tau'} \bar{v}_j + \sum_{\tau_j \in \tau'} \bar{v}_j = f(\mathbf{p}^*)+ \sum_{\tau_j \in \tau'} \bar{v}_j.
\end{equation}
Using \eqref{thmeq-poa2b}, \eqref{thmeq-poa3}, and \eqref{thmeq-poa4}, we obtain
\begin{equation}
\label{thmeq-poa4a}
\max\limits_{\mathbf{p} \in A^*}f(\mathbf{p}) \leq  f(\mathbf{p}^*)+ (m-n) \max_{\tau_j \in \tau'}\bar{v}_j \leq f(\mathbf{p}^*) + \frac{m-n}{n}f(\mathbf{p}^*),
\end{equation}
which implies
$\max\limits_{\mathbf{p} \in A^*}f(\mathbf{p})/f(\mathbf{p}^*) \leq m/n.$ Note that our analyses of Cases 1 and 2 together imply \eqref{thmeq-poa1}}.
\qed \end{proof}
In light of Theorem \ref{POA-T}, for single-station systems with simple tasks, $\textnormal{PoA}(\Gamma_{\textnormal{DTE}})$ is bounded from above based on the number of tasks and robots, irrespective of the task values $\bar{v}_i$. While there can also be many other special cases where $\text{PoA}(\Gamma_{\text{DTE}})$ is guaranteed to be small, in general it can be arbitrarily large when the task specifications do not have a special structure as we have shown in Example 3. Accordingly, noisy best-response learning algorithms such as LLL are needed to ensure near-optimal performance in $\Gamma_{\text{DTE}}$ when there is no prior information on task specifications that indicates a small $\text{PoA}(\Gamma_{\text{DTE}})$.
{ \subsection{Convergence Rate} The proposed approach is based on providing the robots with the specifications of tasks in the upcoming episode and having them optimize their own trajectories in a distributed manner via learning in the the corresponding game $\Gamma_{DTE}$. In practice, the robots have a limited amount of time to plan their trajectories before each episode. 
Hence, in addition to the limiting behavior of the learning process, another important consideration is its convergence rate.

 Convergence time to a Nash equilibrium under BR depends on the type of potential game. There exist examples for both exponential (e.g.,\cite{durand2016complexity}) and polynomial (e.g., \cite{even2003convergence, babichenko2016graphical}) growth of convergence time in the number of agents. On the other hand, the rate of convergence to the limiting distribution $\mu_{\epsilon}$ under LLL depends on both the type of game and the noise parameter $\epsilon$. As $\epsilon$ gets closer to zero, the mass of the limiting distribution $\mu_\epsilon$ accumulates on the  potential maximizers (optimal Nash equilibria) as per \eqref{sstab}. However,  the convergence to $\mu_\epsilon$ becomes slower as $\epsilon$ decreases since it also becomes harder for the agents to leave suboptimal Nash equilibria. In other words, there is an inherent trade-off between the efficiency of the limiting distribution $\mu_\epsilon$ and the time it takes to approach $\mu_\epsilon$, which can be tuned via the noise parameter $\epsilon$.  Similar to BR, LLL also has examples of potential games for both exponential and polynomial (e.g., almost linear) growth of convergence time in the number of agents (e.g.,  \cite{kreindler2013fast,asadpour2009inefficiency, borowski2015fast, shah2010dynamics,ellison1993learning}). Our numerical results in the next section suggest that both algorithms yield a fast convergence rate in the proposed game $\Gamma_{DTE}$. While BR sometimes produces very suboptimal solutions, LLL typically produces a near-optimal solution in a reasonable amount of time.

}

{ \section{Simulation Results}
\label{sims}
We consider the environment in Fig. 1 and present two case studies to numerically demonstrate the performance of the proposed distributed planning approach based on game-theoretic learning. 

\subsection{Case Study 1}
In the first case study, we demonstrate and compare the performances under BR and LLL in a sample scenario. We consider a team of 10 robots that are allocated to the stations as follows: 4 robots at $s_1$, 4 robots at $s_2$, and 2 robots at $s_3$. Robots plan their own trajectories over an episode of length $T=8$. In this episode, there are seven tasks whose specifications are provided in Table \ref{Tab-tasks}. Each task has a value function with one of the two following structures:

\begin{equation}
    \label{sim-vex1}
    v_i(\mathbf{c}_i(\mathbf{p}))=\left\{\begin{array}{ll} \bar{v}_i, \mbox{ if $\max (\mathbf{c}_i(\mathbf{p})) \geq c^*_i$,} \\ 0,\mbox{ otherwise,}\end{array}\right., 
\end{equation}
\begin{equation}
    \label{sim-vex2}
    v_i(\mathbf{c}_i(\mathbf{p}))=\left\{\begin{array}{ll} \bar{v}_i, \mbox{ if $\mathbf{1}^{\text{T}}\mathbf{c}_i(\mathbf{p}) \geq c^*_i$,} \\ 0,\mbox{ otherwise,}\end{array}\right.,
\end{equation}
where the parameters $\bar{v}_i$ and $c^*_i$ for each task are provided in Table \ref{Tab-tasks}. The value function $v_i$ has the structure in \eqref{sim-vex1} for $i \in \{3,5\}$ and the structure in \eqref{sim-vex2} for  $i \in \{1,2,4,6,7\}$. 

We first investigate the performance under LLL by randomly generating 100 runs. In each run,  the robots start the distributed planning process with randomly selected initial trajectories and follow the LLL algorithm with $\epsilon=0.2$ for 300 rounds. In this scenario, given the episode length $T=8$, the number of feasible trajectories that start and end at the corresponding stations are as follows: 405417 (station $s_1$), 161708 (station $s_2$), and  9254 (station $s_3$). Designing the action sets as per \eqref{actset} significantly reduces the number of trajectories for each agent. The resulting number of trajectories in each action set is as follows: 39 (station $s_1$), 16 (station $s_2$), and  18 (station $s_3$). Note that this significant reduction in the number of trajectories greatly improves the scalability of the proposed approach by reducing not only the number of rounds required to observe the limiting behavior but also the computation time required in each round by the updating agent. On a computer with Core i7 – 2.2 GHz CPU and 16 GB RAM, the computation of each action set as per \eqref{actset} takes approximately 10 seconds whereas executing the learning algorithm for 300 rounds (pick a random agent at each round and update its action as per LLL) takes about 15 seconds.

In  Fig. \ref{siml}, we show the evolution of the total value of tasks that would be completed under the trajectories in each round $k$, i.e.,  $\mathbf{p}(k)$, as the robots keep updating their trajectories over 300 rounds. The top figure in Fig. \ref{siml} shows a sample run, where the robots have reached a joint plan that completes all the tasks and yields a total value of 30 within 140 rounds.  Given the stochasticity of the learning process, the bottom figure in Fig. \ref{siml} aims to illustrate the expected behavior under LLL in this scenario by showing the average (blue solid line) as well as the maximum and minimum (dotted red lines) of the total value at each round based on the results from 100 runs. The maximum possible total value, i.e., 30 from completing all tasks, was first observed in one of the runs in round 11. Furthermore, after round 107, all 100 runs maintain a total value in the interval 25-30 while the average keeps increasing towards 30, i.e., the number of runs maintaining the maximum possible value increases over time.  The average total values at some instances were observed as: 25.85 at round 50, 26.79 at round 100, 27.57 at round 200, and 27.87 at round 300. This trend is consistent with the typical behavior when LLL is used in potential games: agents reach a near-optimal configuration in a relatively short amount of time and they eventually maintain the global optima with high probability.

\begin{table}\scriptsize \centering
{\renewcommand{\arraystretch}{1.2}
\begin{center}
\begin{tabular}{cccccc} \toprule
    {Task ID ($i$)} & {Location ($\ell_i$)} & {Arrival $(t^a_i)$} & {Departure $(t^d_i)$} & {Value $(\bar{v}_i)$} & {Threshold $(c^*_i)$} \\ \midrule
    1  & (3,3) & 1 & 7 & 4 & 6  \\ 
     2  & (2,3) & 0 & 5 & 5 & 6  \\
    3  & (6,5) & 2 & 6 & 3 & 2  \\
   4  & (2,1) & 1 & 7 & 5 & 7  \\
    5  & (4,1) & 3 & 6 & 4 & 2  \\
    6  & (6,2) & 0 & 8 & 5 & 8  \\
     7  & (5,5) & 0 & 8 & 4 & 5  \\
   \bottomrule
\end{tabular}
\end{center}}
\caption{Task specifications for the first case study.}
  \label{Tab-tasks}
\end{table}
\begin{figure}[htb]
\centering
\includegraphics[trim =8mm 0mm 0mm 5mm, clip,scale=0.4]{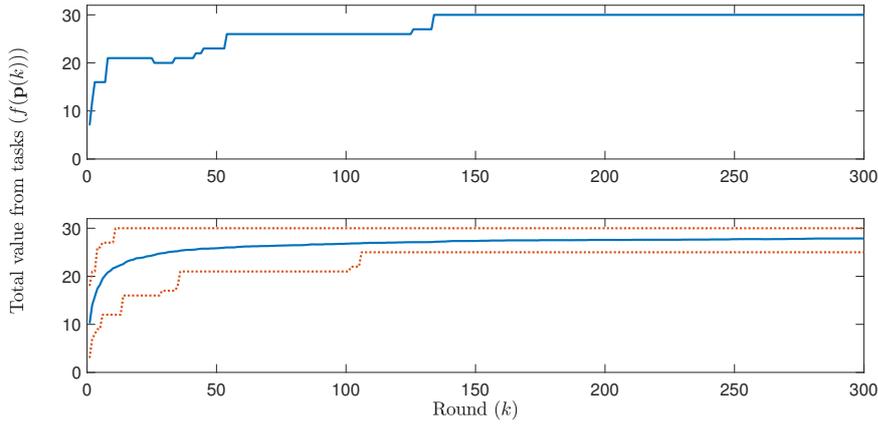}
\caption{ Evolution of the total value in the first case study while the robots update their trajectories by following LLL for 300 rounds. The top figure shows a sample run, the bottom figure shows the average (solid blue line) as well as the maximum and minimum (dotted red lines) values at each round for 100 runs starting with random initial trajectories. }
\label{siml}
\end{figure}

As a second set of simulations, we also utilize the best response (BR) algorithm for the same scenario. Given that the computation time in each round is similar for LLL and BR, we execute the BR algorithm also for 300 rounds. Under such updates, the system rapidly converges to one of the Nash equilibria, usually within the first 10-20 rounds in this example. The resulting equilibrium largely depends on the initial trajectories and which agent updates in each round. { In Fig. \ref{BR-hist} we show the distribution of $f(\mathbf{p})$ at the resulting Nash equilibria for 1000 runs, each of which was started with a random choice of initial trajectories. Thirteen different values of $f(\mathbf{p})$ were observed at the resulting equilibria with the following numbers of occurrences: 11 (1 time), 16 (11 times), 17 (15 times), 18 (9 times), 19 (13 times), 20 (16 times), 21 (127 times), 22 (158 times), 23 (173 times), 25 (105 times) , 26 (270 times), 27 (73 times), and 30 (29 times), which yield an average of 23.78.  In comparison, the runs with LLL resulted in an expected total value of 27.87 at the end of 300 rounds. Furthermore, while BR was observed to result in a significantly suboptimal Nash equilibrium in some cases (e.g., a total value of 11 or 16), none of the runs with LLL produced a total value smaller than 25 as it can be seen in Fig. \ref{siml}. }

\begin{figure}[htb!]
\centering
\includegraphics[trim =8mm 0mm 0mm 0mm, clip,scale=0.4]{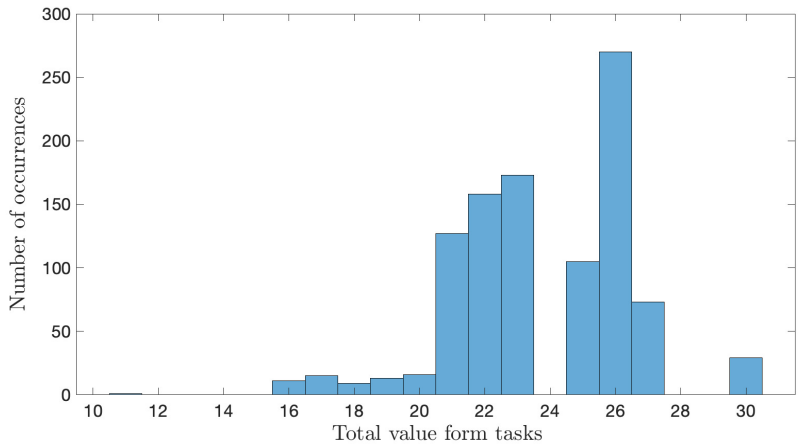}
\caption{ Distribution of the total value of completed tasks $f(\mathbf{p})$ at the Nash equilibria resulting from BR in 200 runs with random initial trajectories.  }
\label{BR-hist}
\end{figure}

\subsection{Case Study 2}
In the second case study, we investigate the performance of LLL in 9 scenarios with different numbers of tasks and robots. To this end, we first create a list of 30 tasks as shown in Table \ref{Tab-tasks2}. Among these 30 tasks, the value function $v_i$ has the structure in \eqref{sim-vex1} for every $i \in \{3,5,12,14,18,23,25,28\}$ and the structure in \eqref{sim-vex2} for all the other tasks.  We consider three different sets of tasks from this list: 1) 10 tasks (1 to 10 in Table \ref{Tab-tasks2}), 2) 20 tasks (1 to 20 in Table \ref{Tab-tasks2}), and 3) 30 tasks (all the tasks in Table \ref{Tab-tasks2}). We also consider three different cases in terms of the number of robots: 1) 5 robots (2 at $s_1$, 2 at $s_2$, and 1 at $s3$), 2) 10 robots (4 at $s_1$, 4 at $s_2$, and 2 at $s3$), and 3) 15 robots (6 at $s_1$, 6 at $s_2$, and 3 at $s3$). By considering all combinations of these sets of tasks and sets of robots, we obtain a total of 9 scenarios where the number of tasks ranges from 10 to 30 and the number of robots ranges from 5 to 15. Given the task specifications and the episode length $T=8$, designing the action sets as per \eqref{actset} leads to the following numbers of trajectories in these scenarios:
\begin{itemize}
    \item {\bf Scenarios with 10 tasks:} 88 trajectories for each robot at $s_1$, 70 trajectories for each robot at $s_2$, and 26 trajectories for each robot at $s_3$
    \item {\bf Scenarios with 20 tasks:} 667 trajectories for each robot at $s_1$, 238 trajectories for each robot at $s_2$, and 84 trajectories for each robot at $s_3$
    \item {\bf Scenarios with 30 tasks:} 686 trajectories for each robot at $s_1$, 415 trajectories for each robot at $s_2$, and 128 trajectories for each robot at $s_3$
\end{itemize}


Since the size of the action space grows exponentially with the number of robots $n$, i.e., $|A_1| \times |A_2| \times \hdots \times |A_n|$, it is clearly intractable to find the maximum feasible value in these scenarios by searching through all possible joint plans (i.e., a centralized solution). The size of action space ranges from approximately $9.86 \times 10^8$ action profiles in the smallest case (5 robots and 10 tasks) to approximately $1.1 \times 10^{39}$ action profiles in the largest case (15 robots and 30 tasks). For each of the 9 scenarios, we generate 10 runs, each of which starts with randomly picked initial trajectories and executes LLL with the noise parameter $\epsilon=0.2$ for a total of 600 rounds. For each set of 10 runs, we provide the average, minimum, and maximum of total value at each round in Figs. \ref{5robsims} (scenarios with 5 robots), \ref{10robsims} (scenarios with 10 robots), and \ref{15robsims} (scenarios with 15 robots). In all of these figures, we see that the minimum and maximum values among the runs rapidly increase and approach each other, which creates a narrow envelope for the average value. In these simulations, the minimum, average, and maximum values of 10 runs at the end of round 600 are observed as follows:
\begin{itemize}
    \item {\bf Scenarios with 5 robots:} 19, 19.7, 20 (10 tasks); 29, 30.1, 31 (20 tasks); 29, 30.1, 31 (30 tasks).
    \item {\bf Scenarios with 10 robots:} 26, 26, 26 (10 tasks); 46, 48.6, 51 (20 tasks); 53, 56.2, 59 (30 tasks).
    \item {\bf Scenarios with 15 robots:} 26, 26, 26 (10 tasks); 58, 59.2, 64 (20 tasks); 66, 74.5, 80 (30 tasks).
\end{itemize}

In 3 of these 9 scenarios, we can easily verify that the maximum value observed is equal to the maximum feasible value since it corresponds to the completion of all the tasks: 10 robots and 10 tasks (total value is 26), 15 robots and 10 tasks, 15 robots and 20 tasks (total value is 64). Furthermore, in each of these 9 scenarios, the average value of 10 runs at the end of 600 rounds is more than 92\% of the maximum value. These results suggest that the robots are expected to obtain near-optimal joint plans by following LLL for 600 rounds in all of these scenarios. Overall, LLL has maintained a fast convergence speed (in 600 rounds) despite the significant growth in the sizes of action spaces, i.e., from approximately $9.86 \times 10^8$ action profiles in the smallest scenario (5 robots and 10 tasks) to approximately $1.1 \times 10^{39}$ action profiles in the largest scenario (15 robots and 30 tasks). 




\begin{table}\scriptsize \centering
{\renewcommand{\arraystretch}{1.2}
\begin{center}
\begin{tabular}{cccccc} \toprule
    {Task ID ($i$)} & {Location ($\ell_i$)} & {Arrival $(t^a_i)$} & {Departure $(t^d_i)$} & {Value $(\bar{v}_i)$} & {Threshold $(c^*_i)$} \\ \midrule
    1  & (3,3) & 1 & 7 & 4 & 6  \\ 
     2  & (2,3) & 0 & 5 & 3 & 2  \\
    3  & (6,5) & 2 & 6 & 3 & 2  \\
   4  & (2,1) & 1 & 7 & 2 & 2  \\
    5  & (4,1) & 3 & 6 & 3 & 2  \\
    6  & (6,2) & 0 & 8 & 2 & 2  \\
     7  & (5,5) & 0 & 8 & 4 & 4  \\
     8  & (7,4) & 3 & 8 & 2 & 2 \\
         9  & (1,2) & 0 & 5 & 1 & 1  \\ 
     10  & (5,2) & 5 & 8 & 2 & 1  \\
    11  & (7,2) & 2 & 5 & 3 & 2  \\
   12  & (3,2) & 3 & 7 & 6 & 3  \\
    13  & (3,4) & 0 & 4 & 3 & 3  \\
    14  & (1,1) & 0 & 4 & 3 & 2  \\
     15  & (5,1) & 2 & 5 & 3 & 3  \\
     16  & (3,4) & 4 & 7 & 5 & 2 \\
         17  & (3,1) & 2 & 8 & 5 & 2  \\ 
     18  & (1,3) & 0 & 8 & 6 & 3  \\
    19  & (7,4) & 0 & 3 & 2 & 1  \\
   20  & (6,4) & 0 & 4 & 2 & 1  \\
     21  & (7,3) & 1 & 3 & 3 & 4  \\
   22  & (7,3) & 4 & 8 & 2 & 2  \\
    23  & (2,3) & 5 & 7 & 3 & 8  \\
    24  & (1,2) & 5 & 8 & 4 & 4  \\
     25  & (3,5) & 1 & 3 & 2 & 4  \\
     26  & (5,2) & 2 & 5 & 5 & 5 \\
         27  & (3,5) & 5 & 8 & 2& 4  \\ 
     28  & (1,1) & 4 & 8 & 2 & 3  \\
    29  & (6,4) & 4 & 6 & 1 & 2  \\
   30  & (6,4) & 6 & 8 & 1 & 2  \\
   \bottomrule
\end{tabular}
\end{center}}
\caption{Task specifications for the second case study.}
  \label{Tab-tasks2}
\end{table}

\begin{figure}[htb]
\centering
\includegraphics[trim =20mm 0mm 0mm 5mm, clip,scale=0.36]{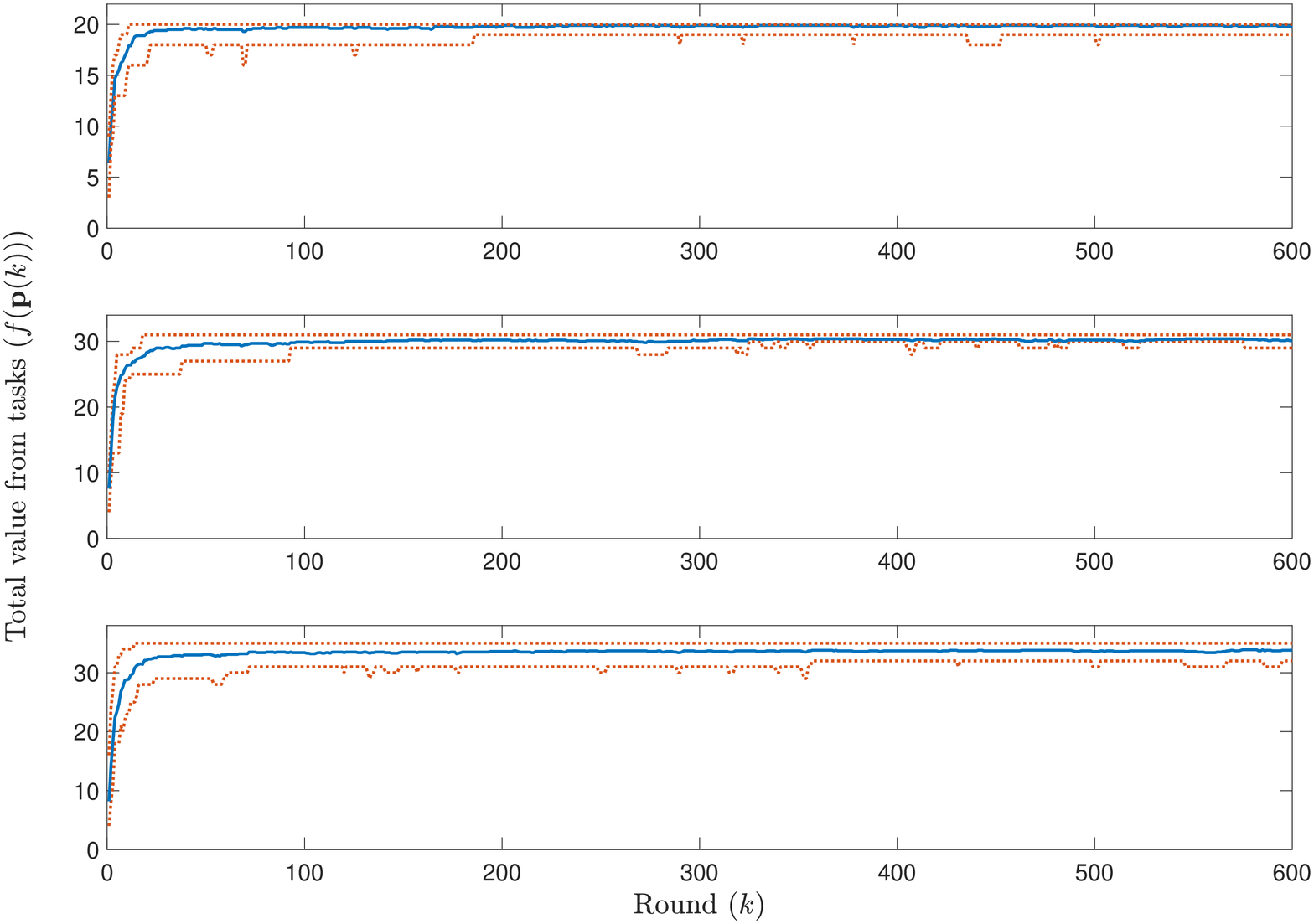}
\caption{Evolution of the total value under LLL in the second case study for the scenarios with 5 robots. Figures correspond to 10 (top), 20 (middle), and 30 (bottom) tasks. Each figure shows the average (solid blue line) as well as the maximum and minimum (dotted red lines) values at each round for 10 runs.  }
\label{5robsims}
\end{figure}

\begin{figure}[htb]
\centering
\includegraphics[trim =20mm 0mm 0mm 5mm, clip,scale=0.36]{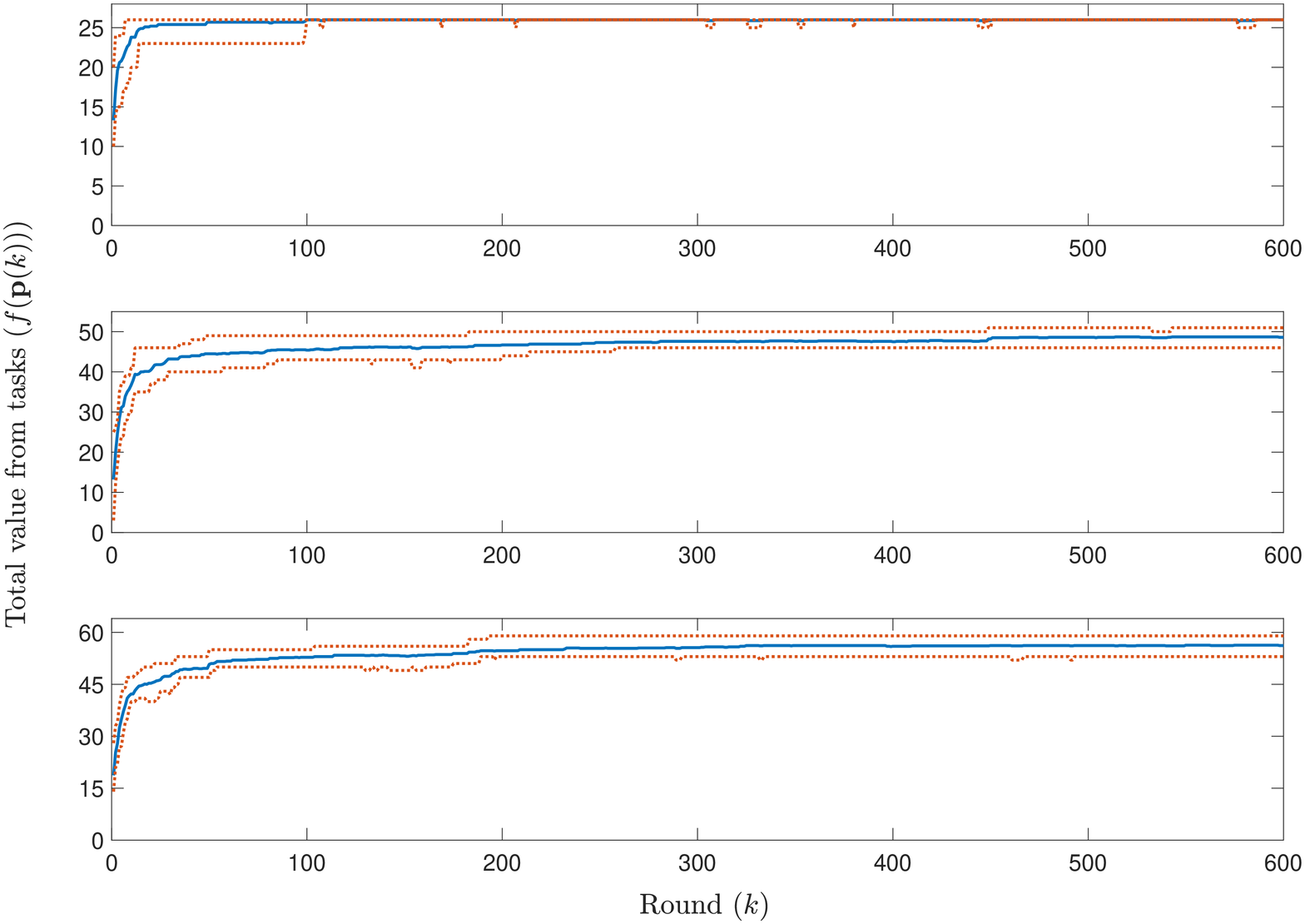}
\caption{Evolution of the total value under LLL in the second case study for the scenarios with 10 robots. Figures correspond to 10 (top), 20 (middle), and 30 (bottom) tasks. Each figure shows the average (solid blue line) as well as the maximum and minimum (dotted red lines) values at each round for 10 runs.  }
\label{10robsims}
\end{figure}

\begin{figure}[htb]
\centering
\includegraphics[trim =20mm 0mm 0mm 5mm, clip,scale=0.36]{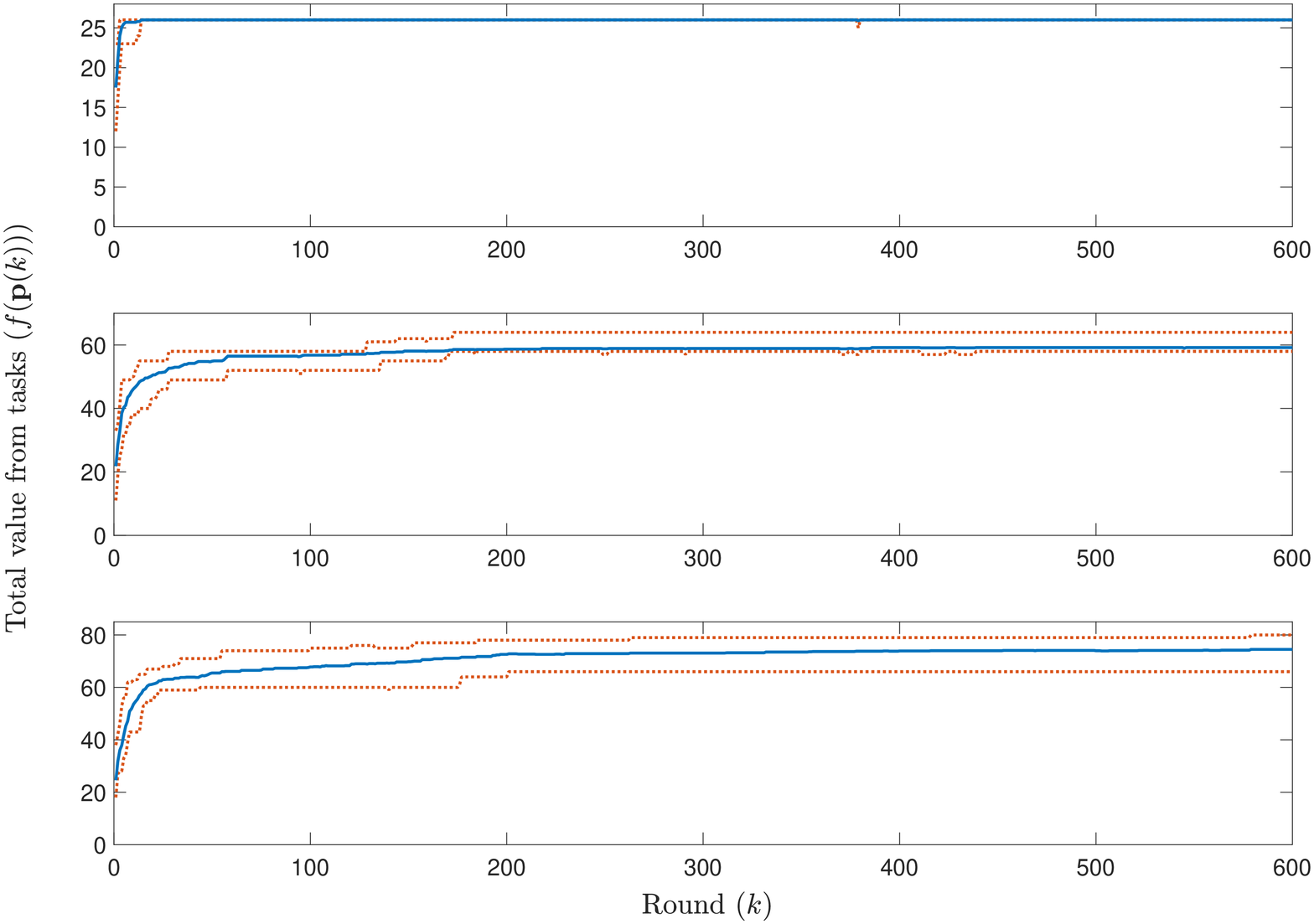}
\caption{ Evolution of the total value under LLL in the second case study for the scenarios with 15 robots. Figures correspond to 10 (top), 20 (middle), and 30 (bottom) tasks. Each figure shows the average (solid blue line) as well as the maximum and minimum (dotted red lines) values at each round for 10 runs. }
\label{15robsims}
\end{figure}
}

\section{Experimental Results} 
\label{exp-res}

We present the results of experiments with a team of three Crazyflies 2.0 in a $3m\times3m\times1.5m$ motion-capture space equipped with a \texttt{VICON} system with eight cameras. 
We use the Crazyswarm\footnote{\url{https://github.com/USC-ACTLab/crazyswarm}} package \cite{preiss2017crazyswarm} to run the low-level control algorithms and to link the \texttt{VICON} system with the Crazyflies. The experiment is performed by using a desktop computer with 4 cores running Ubuntu 16.04, 4.0GHz CPU, and 32GB RAM.

The experiment is designed as a small scale representation of an aerial monitoring scenario in the environment shown in Fig. \ref{env}, where the obstacles correspond to no-fly zones. Three drones $\{r_1,r_2,r_3\}$,  each of which is assigned to a different station ($r_1$ at $s_1$, $r_2$ at $s_2$, and $r_3$ at $s_3$),  optimize their trajectories to serve the incoming monitoring/surveillance tasks. In this experiment, each drone is assigned to a different altitude to avoid potential collisions. We consider episodes of length $T=8$ for moving between regions plus two additional time steps for take-off and landing (one time step for each). Each time step corresponds to two seconds in real time. Accordingly, each episode implies a flight time of 20 seconds for the drones (16 seconds for traversing their trajectories over the episode horizon of $T=8$ steps and 4 seconds for take-off/landing). The experiment consists of five episodes, each of which involved a subset of the eight tasks whose specifications are listed in Table \ref{Tab-tasks2}. The value function $v_i$ has the structure in \eqref{sim-vex1} for $i \in \{3,5\}$ and the structure in \eqref{sim-vex2} for  $i \in \{1,2,4,6,7,8\}$. For an aerial monitoring application, the value function in \eqref{sim-vex1} may correspond to a task that requires multiple aerial images taken simultaneously from different viewpoints, and the value function in \eqref{sim-vex2} may correspond to a task that does not require such a simultaneity in the images. 

The tasks arriving in each episode and the resulting number of trajectories in the action sets in \eqref{actset} for each robot are provided in Table \ref{episodes}. For each episode, starting with randomly selected initial trajectories, the drones update their trajectories by following LLL for a period of 50 rounds. { The evolution  of the total value from tasks as the robots update their trajectories during the learning process is shown in Fig. \ref{exp-LLL}. The trajectories obtained at the end of 50 rounds are provided in Table \ref{Tab-tasks3}. Based on the specifications of the tasks in each episode as given in Tables \ref{Tab-tasks2} and \ref{episodes}, it can be verified that all the tasks are completed under the trajectories in Table \ref{Tab-tasks3}. Accordingly, the total values of completed tasks are as follows: 11 (episode 1), 11 (episode 2), 10 (episode 3), 12 (episode 4), 10 (episode 5). In all the episodes, the maximum possible value was reached within the first 12 rounds of learning and maintained throughout the remaining rounds.  }

\begin{figure}[htb]
\centering
\includegraphics[trim =8mm 0mm 0mm 5mm, clip,scale=0.55]{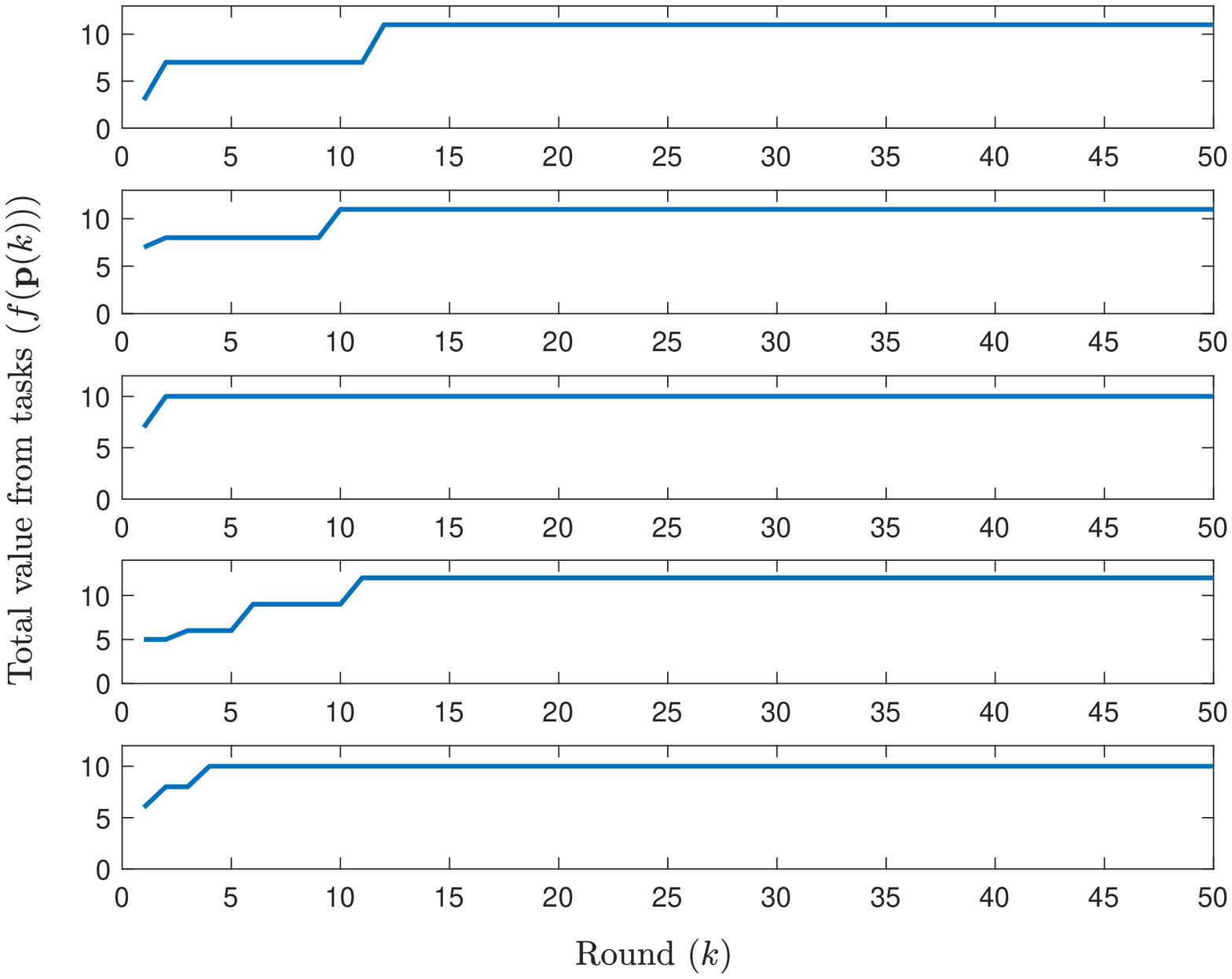}
\caption{ Total value of completed tasks over 50 rounds as the robots update their trajectories by following log-linear learning (LLL). Plots correspond to the five episodes (ordered from top to bottom) in the experiments. }
\label{exp-LLL}
\end{figure}


\begin{table}[htb!] \scriptsize \centering

{\renewcommand{\arraystretch}{1.2}
\begin{center} 
\begin{tabular}{ccccc} \toprule
    {Episode} & {Task IDs} & {$|A_1|$} & {$|A_2|$} & {$|A_3|$}  \\ \midrule
    1  & \{1,2,6,8\} & 10 & 8 & 7  \\ 
     2  & \{1,3,7\} & 7 & 6 & 8  \\
    3  & \{2,4,5,6\} & 16 & 6 & 1   \\
   4  & \{2,3,4,7\}& 8 & 6 & 11   \\
    5  & \{1,4,6,8\} & 8 & 8 & 2   \\
   \bottomrule
\end{tabular}
\end{center}
}
\caption{Tasks in each of the five episodes (see Table \ref{Tab-tasks2} for the specifications) and the resulting number trajectories in the action sets in \eqref{actset}.}
  \label{episodes}
\end{table}

\begin{table}[htb!]
\scriptsize \centering

{ \renewcommand{\arraystretch}{1.2}
\begin{center}
\begin{tabular}{c} \toprule
    {Episode 1} \\ \midrule
    \mbox{$\mathbf{p}_1=\{(2,2), (2,3), (2,3), (2,3), (3,3), (3,3), (3,3),(3,3), (2,2)\} $} \\ 
   \mbox{$  \mathbf{p}_2=\{(6,3), (6,2), (6,2), (6,2), (6,3), (7,4), (7,4),(7,4), (6,3)\}$}\\
    \mbox{$ \mathbf{p}_3=\{(4,5), (3,4), (3,3), (3,3), (3,3), (3,3), (3,3),(3,4), (4,5)\}  $}\\
    \midrule
       {Episode 2} \\ \midrule
   \mbox{$ \mathbf{p}_1=\{(2,2), (3,3), (3,3), (3,3), (3,3), (3,3), (3,3),(3,3), (2,2)\} $} \\ 
    \mbox{$ \mathbf{p}_2=\{(6,3), (6,4), (5,5), (5,5), (5,5), (6,5), (6,5),(6,4), (6,3)\}$}\\
   \mbox{$  \mathbf{p}_3=\{(4,5), (5,5), (5,5), (5,5), (6,5), (6,5), (6,5),(5,5), (4,5)\} $} \\
    \midrule   {Episode 3} \\ \midrule
   \mbox{$ \mathbf{p}_1=\{(2,2), (2,1), (2,1), (2,1), (3,1), (4,1), (4,1),(3,1), (2,2)\} $} \\ 
  \mbox{$   \mathbf{p}_2=\{(6,3), (6,2), (6,2), (6,2), (5,1), (4,1), (4,1),(5,2), (6,3)\}$}\\
  \mbox{$   \mathbf{p}_3=\{(4,5), (3,4), (2,3), (2,3), (2,3), (2,3), (2,3),(3,4), (4,5)\}  $}\\
    \midrule   {Episode 4} \\ \midrule
   \mbox{$ \mathbf{p}_1=\{(2,2), (2,3), (2,3), (2,3), (1,2), (2,1), (2,1),(2,1), (2,2)\} $} \\ 
  \mbox{$   \mathbf{p}_2=\{(6,3), (6,4), (6,5), (6,5), (6,5), (5,5), (5,5),(6,4), (6,3)\}$}\\
   \mbox{$  \mathbf{p}_3=\{(4,5), (5,5), (6,5), (6,5), (5,5), (5,5), (5,5),(5,5), (4,5)\} $} \\
    \midrule   {Episode 5} \\ \midrule
   \mbox{$ \mathbf{p}_1=\{(2,2), (3,3), (3,3), (3,3), (2,2), (2,1), (2,1),(2,1), (2,2)\} $} \\ 
   \mbox{$  \mathbf{p}_2=\{(6,3), (6,2), (6,2), (6,2), (6,3), (7,4), (7,4),(7,4), (6,3)\}$}\\
  \mbox{$   \mathbf{p}_3=\{(4,5), (3,4), (3,3), (3,3), (3,3), (3,3), (3,3),(3,4), (4,5)\}$}  \\
   \bottomrule
\end{tabular}
\end{center}
}
\caption{Trajectories obtained via LLL in each episode of the experiment. }
  \label{Tab-tasks3}
\end{table}

Some instances from the experiments are shown in Figures \ref{fig:real_experiments_1} and \ref{fig:real_experiments_2}. The drones indicated by arrows start the mission at their stations. The first episode involves tasks 1, 2, 6, and 8. As per the trajectories in Table \ref{Tab-tasks3}, initially $r_1$ serves task 2, $r_2$ serves task 6 and $r_3$ serves task 1 as highlighted in Fig.~\ref{fig:sfig2}. Once task 2 is completed, $r_1$ joins $r_3$ in serving task 1. At the same time, $r_2$ finishes task 6 and moves to its next destination to serve task 8 as shown in Fig.~\ref{fig:sfig3}. The drones complete all the tasks and return to their stations by the end of the first episode as shown in Fig.~\ref{fig:sfig4}. The second episode involves tasks 1, 3, and 7. First tasks 1 and 7 arrive. Task 1 is served by $r_1$. Task 7 is served by $r_3$ for one time step and then $r_2$ joins $r_3$ to complete the task faster as shown in Fig.~\ref{fig:sfig6}. After completing task 7, $r_2$ and $r_3$ together serve task 3, which requires at least two drones to be simultaneously present for one time step, while $r_1$ continues serving task 1 as shown in Fig.~\ref{fig:sfig7}. Finally, the drones complete tasks 1 and 3 and go back to their stations as shown in Fig.~\ref{fig:sfig8}.

\begin{figure}[ht]
\centering
\begin{subfigure}{.45\textwidth}
  \centering
  \includegraphics[width=\linewidth]{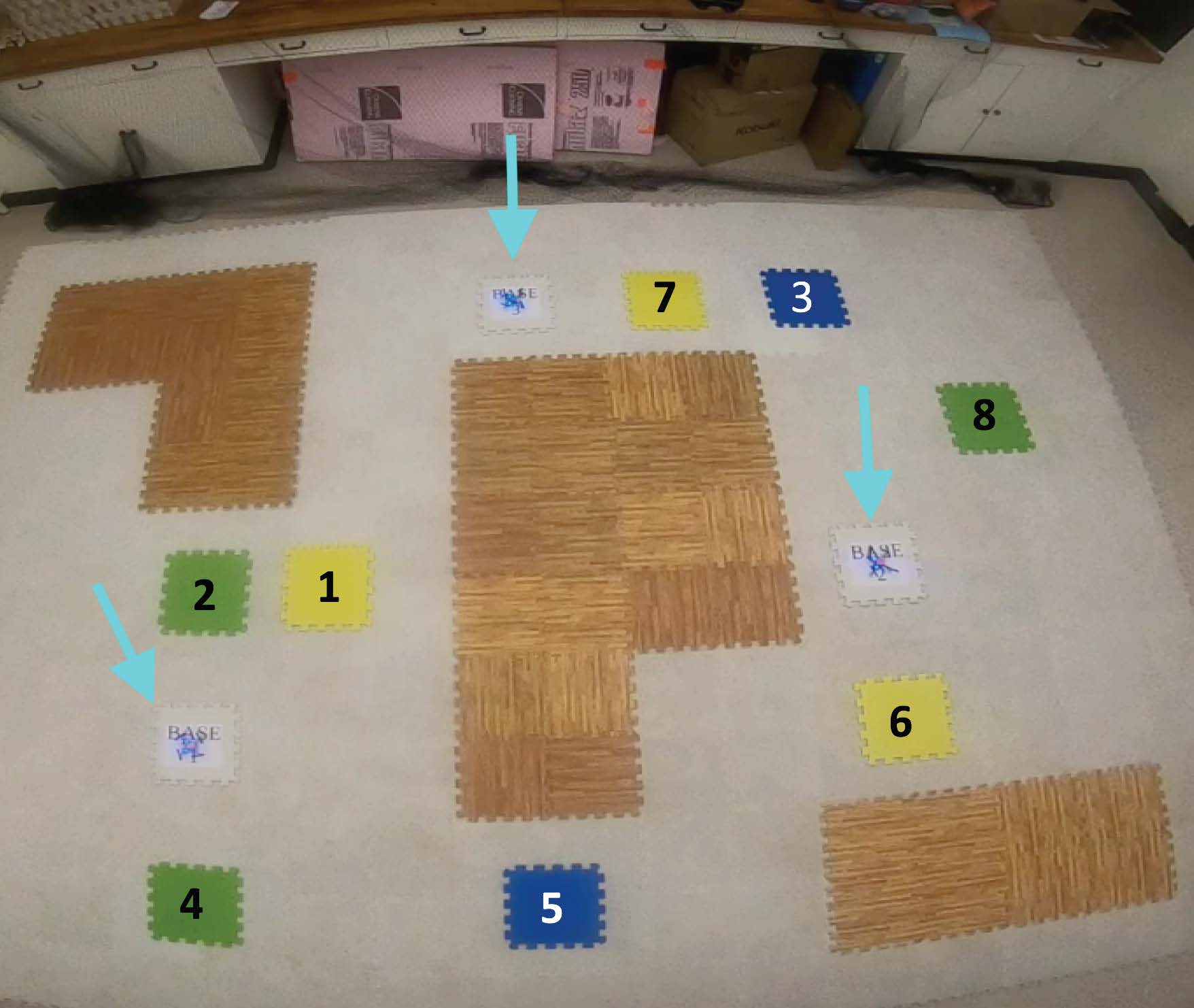}
  \caption{$time=0$ sec}
  \label{fig:sfig1}
\end{subfigure}%
\quad
\begin{subfigure}{.45\textwidth}
  \centering
  \includegraphics[width=\linewidth]{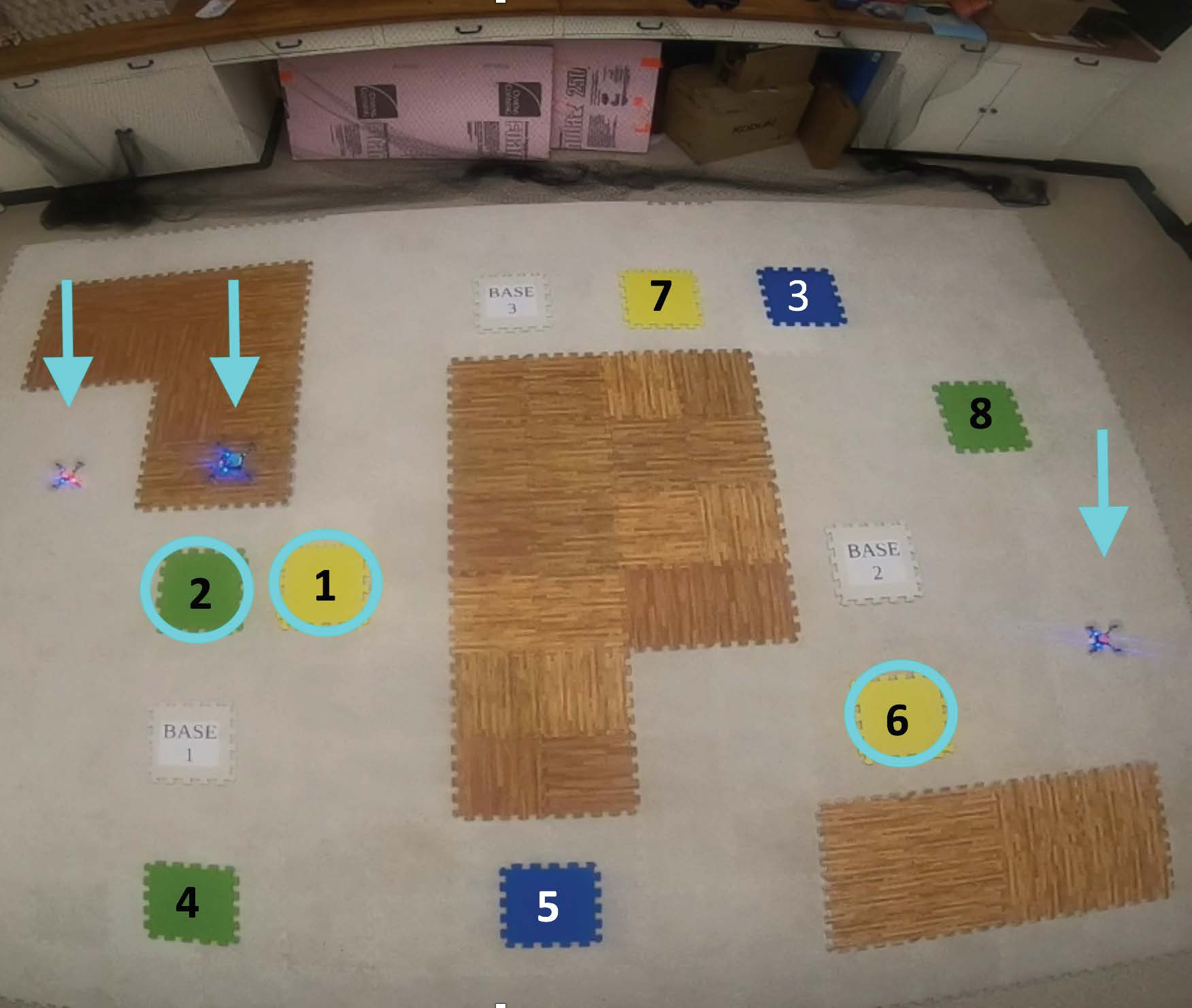}
  \caption{$time=6$ sec}
  \label{fig:sfig2}
\end{subfigure}

\begin{subfigure}{.45\textwidth}
  \centering
  \includegraphics[width=\linewidth]{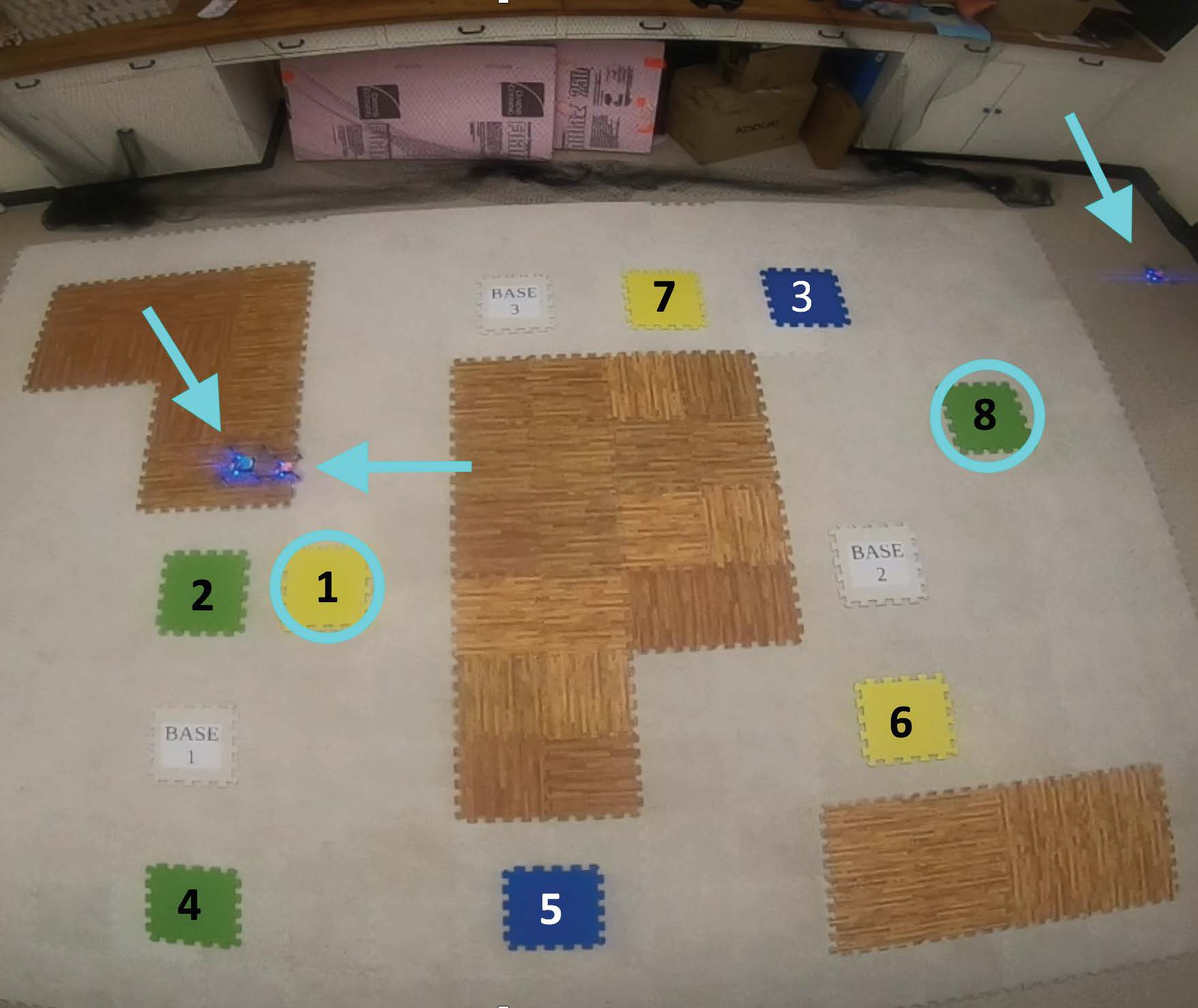}
  \caption{$time=10$ sec}
  \label{fig:sfig3}
\end{subfigure}%
\quad
\begin{subfigure}{.45\textwidth}
  \centering
  \includegraphics[width=\linewidth]{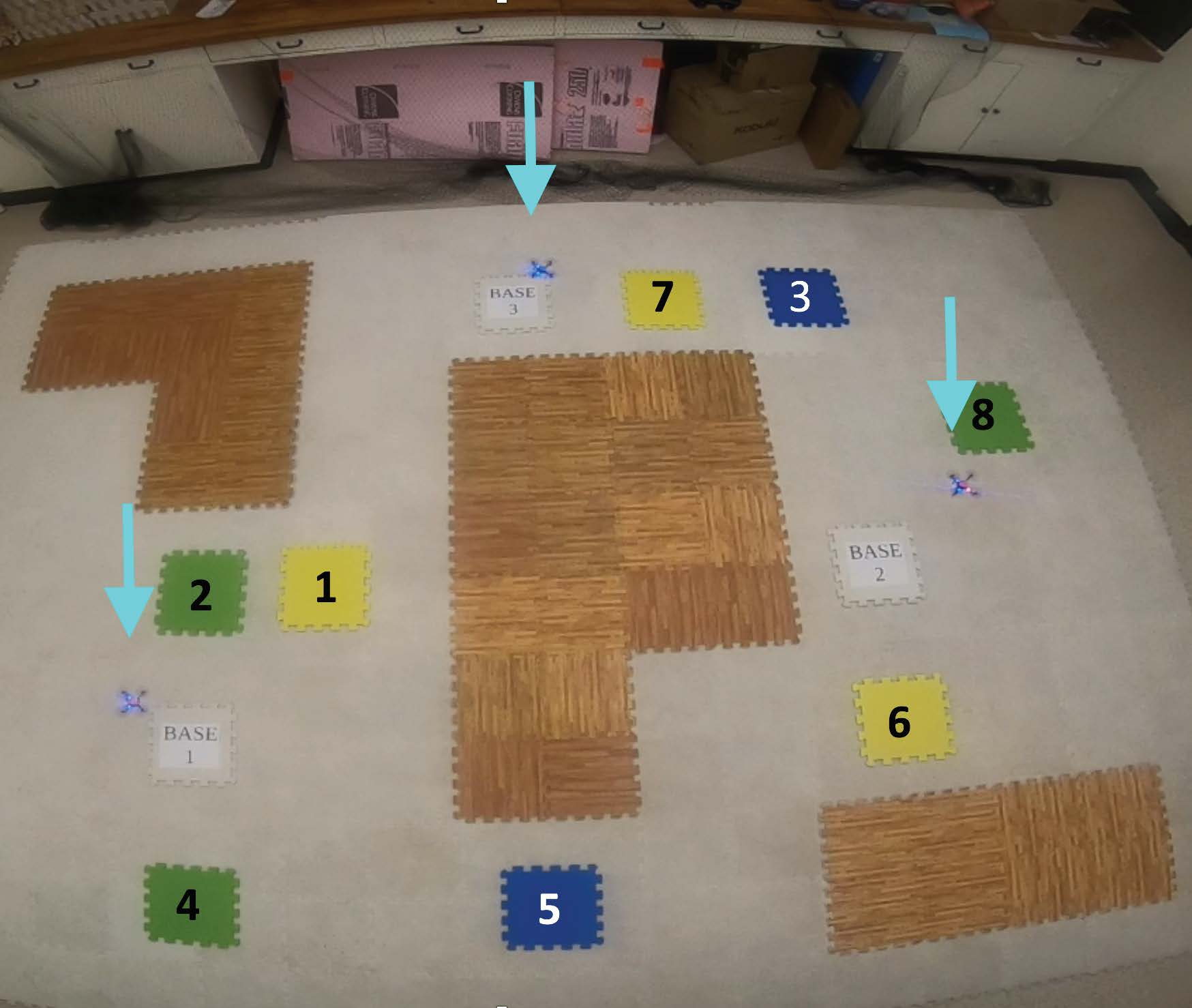}
  \caption{$time=20$ sec}
  \label{fig:sfig4}
\end{subfigure}
\caption{Some instances from episode 1 including 3 drones that are indicated by arrows. In this episode, tasks 1, 2, 6, and 8 arrive and the tasks being served are highlighted in circle. }
\label{fig:real_experiments_1}
\end{figure}

\begin{figure}[htb!]
\centering
\begin{subfigure}{.45\textwidth}
  \centering
  \includegraphics[width=\linewidth]{17.jpg}
  \caption{$time=0$ sec}
  \label{fig:sfig5}
\end{subfigure}%
\quad
\begin{subfigure}{.45\textwidth}
  \centering
  \includegraphics[width=\linewidth]{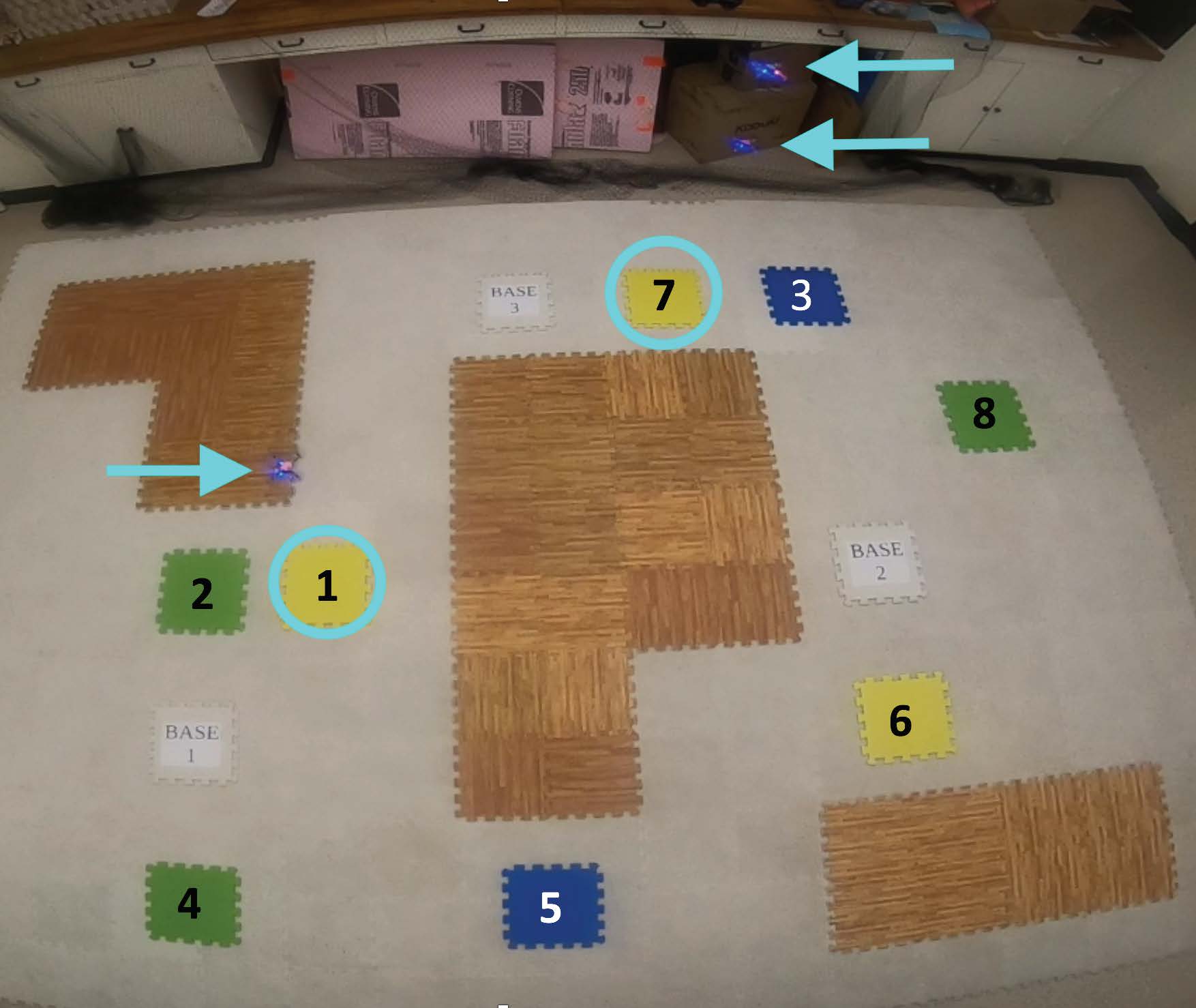}
  \caption{$time=6$ sec}
  \label{fig:sfig6}
\end{subfigure}
\begin{subfigure}{.45\textwidth}
  \centering
  \includegraphics[width=\linewidth]{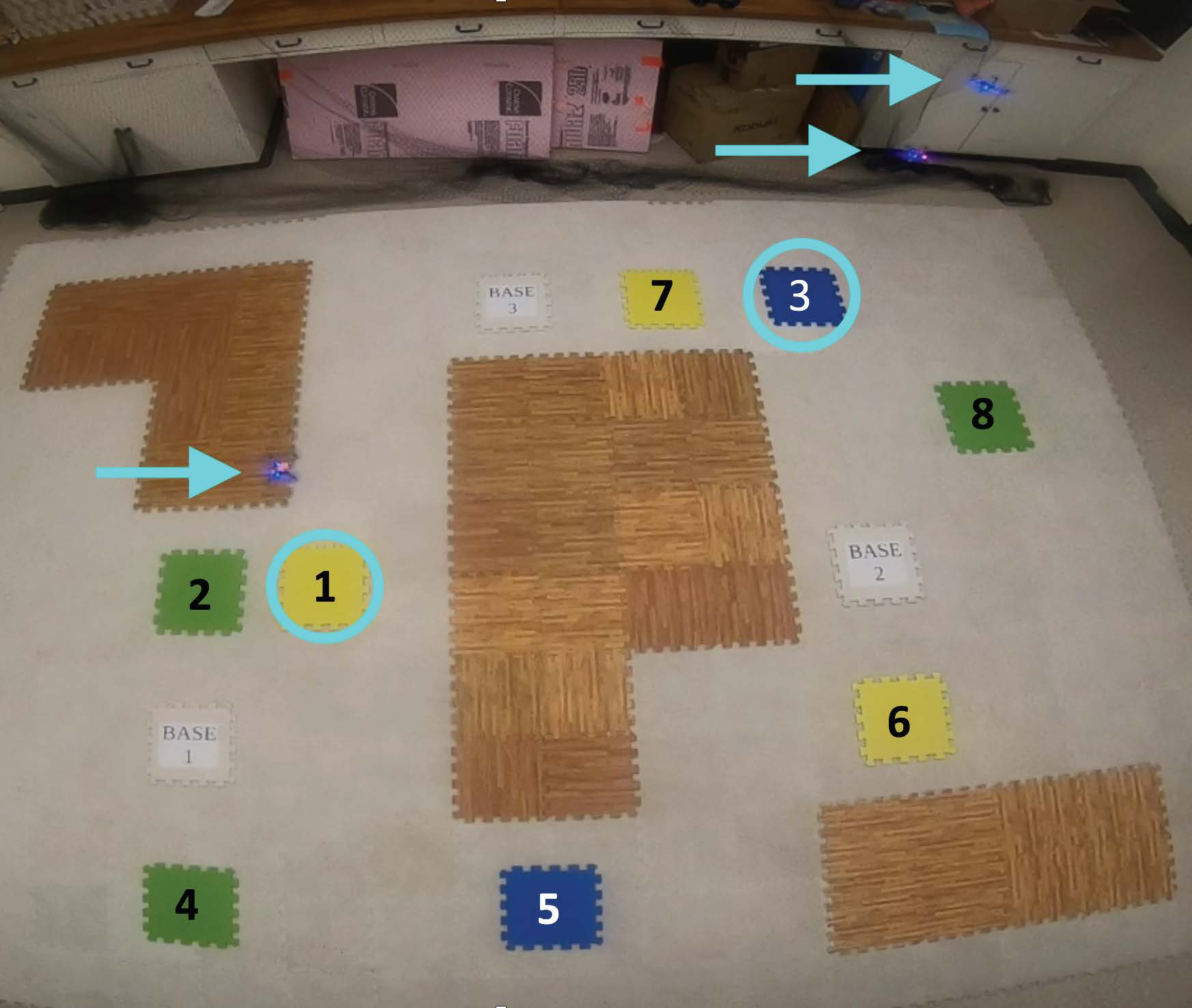}
  \caption{$time=12$ sec}
  \label{fig:sfig7}
\end{subfigure}%
\quad
\begin{subfigure}{.45\textwidth}
  \centering
  \includegraphics[width=\linewidth]{17.jpg}
  \caption{$time=20$ sec}
  \label{fig:sfig8}
\end{subfigure}
\caption{Some instances from episode 2 including 3 drones that are indicated by arrows. In this episode, tasks 1, 3, and 7 arrive and the tasks being served are highlighted in circle. }
\label{fig:real_experiments_2}
\end{figure}

\section{Conclusion}
\label{conc}
In this paper, we presented a game-theoretic approach to the distributed planning of multi-robot systems for providing optimal service to cooperative tasks that are dispersed over space and time. In this setting, each task requires service by sufficiently many robots at the specified location within the specified time window. The robots plan their own trajectories in each episode based on the specifications of incoming tasks. Each robot is required to start and end every episode at its assigned station. We mapped this planning problem to a potential game by setting the utility of each robot to its marginal contribution to the total value from tasks. We presented a systematic way to use the task specifications for designing minimal action sets that include globally optimal joint plans and facilitate fast learning. We then showed that the resulting game can in general have arbitrarily poor Nash equilibria. We also described some special cases where all the equilibria are guaranteed to have bounded suboptimality. We showed how game theoretic learning algorithms such as the best response or the log linear learning can be used to iteratively obtain optimal joint plans in this setting. The performance of the proposed approach was demonstrated via simulations and experimental results.

As a future direction, we plan to extend our work to heterogeneous teams where the robots may have different capabilities and dynamics. Furthermore, we plan to investigate the use of temporal logics for encoding complex cooperative tasks (value functions) to be used in the proposed distributed  planning framework. We are also interested in exploring the use of other distributed learning and optimization algorithms to solve the proposed planing problem. Incorporating additional performance objectives such as the robustness of the total value from tasks to the deviations from generated plans (e.g., due to disturbances or stochastic dynamics) is another direction we intend to pursue.

 \bibliographystyle{spmpsci}  
\bibliography{MyReferences} 

\begin{thebibliography}{10}
\providecommand{\url}[1]{{#1}}
\providecommand{\urlprefix}{URL }
\expandafter\ifx\csname urlstyle\endcsname\relax
  \providecommand{\doi}[1]{DOI~\discretionary{}{}{}#1}\else
  \providecommand{\doi}{DOI~\discretionary{}{}{}\begingroup
  \urlstyle{rm}\Url}\fi

\bibitem{aksaray2016dynamic}
Aksaray, D., Vasile, C.I., Belta, C.: Dynamic routing of energy-aware vehicles
  with temporal logic constraints.
\newblock In: IEEE International Conference on Robotics and Automation (ICRA),
  pp. 3141--3146 (2016)

\bibitem{andersson2006multiprocessor}
Andersson, B., Tovar, E.: Multiprocessor scheduling with few preemptions.
\newblock In: IEEE International Conference on Embedded and Real-Time Computing
  Systems and Applications, pp. 322--334. IEEE (2006)

\bibitem{Arsie09}
Arsie, A., Savla, K., Frazzoli, E.: Efficient routing algorithms for multiple
  vehicles with no explicit communications.
\newblock IEEE Transactions on Automatic Control \textbf{54}(10), 2302--2317
  (2009)

\bibitem{Arslan07}
Arslan, G., Marden, J., Shamma, J.S.: Autonomous vehicle-target assignment: a
  game theoretical formulation.
\newblock ASME Journal of Dynamic Systems, Measurement, and Control pp.
  584--596 (2007)

\bibitem{asadpour2009inefficiency}
Asadpour, A., Saberi, A.: On the inefficiency ratio of stable equilibria in
  congestion games.
\newblock In: International Workshop on Internet and Network Economics, pp.
  545--552. Springer (2009)

\bibitem{babichenko2016graphical}
Babichenko, Y., Tamuz, O.: Graphical potential games.
\newblock Journal of Economic Theory \textbf{163}, 889--899 (2016)

\bibitem{bennewitz2001optimizing}
Bennewitz, M., Burgard, W., Thrun, S.: Optimizing schedules for prioritized
  path planning of multi-robot systems.
\newblock In: IEEE International Conference on Robotics and Automation, vol.~1,
  pp. 271--276 (2001)

\bibitem{Bhat19}
Bhat, R., Yaz{\i}c{\i}o{\u{g}}lu, Y., Aksaray, D.: Distributed path planning
  for executing cooperative tasks with time windows.
\newblock IFAC-PapersOnLine \textbf{52}(20), 187--192 (2019)

\bibitem{Bhattacharya10}
Bhattacharya, S., Likhachev, M., Kumar, V.: Multi-agent path planning with
  multiple tasks and distance constraints.
\newblock In: IEEE International Conference on Robotics and Automation, pp.
  953--959 (2010)

\bibitem{Blume93}
Blume, L.E.: The statistical mechanics of strategic interaction.
\newblock Games and Economic Behavior \textbf{5}(3), 387--424 (1993)

\bibitem{borowski2015fast}
Borowski, H., Marden, J.R.: Fast convergence in semianonymous potential games.
\newblock IEEE Transactions on Control of Network Systems \textbf{4}(2),
  246--258 (2015)

\bibitem{Boyd06}
Boyd, S., Ghosh, A., Prabhakar, B., Shah, D.: Randomized gossip algorithms.
\newblock IEEE Transactions on Information Theory \textbf{52}(6), 2508--2530
  (2006)

\bibitem{Boyd11}
Boyd, S., Parikh, N., Chu, E., Peleato, B., Eckstein, J.: Distributed
  optimization and statistical learning via the alternating direction method of
  multipliers.
\newblock Found. and Trends{\textregistered} in Mach. Learn. \textbf{3}(1),
  1--122 (2011)

\bibitem{braysy2005vehicle}
Br{\"a}ysy, O., Gendreau, M.: Vehicle routing problem with time windows, part
  {I}: Route construction and local search algorithms.
\newblock Transportation science \textbf{39}(1), 104--118 (2005)

\bibitem{Bu08}
Bu, L., Babu, R., De~Schutter, B., et~al.: A comprehensive survey of multiagent
  reinforcement learning.
\newblock IEEE Transactions on Systems, Man, and Cybernetics, Part C
  (Applications and Reviews) \textbf{38}(2), 156--172 (2008)

\bibitem{bullo2011dynamic}
Bullo, F., Frazzoli, E., Pavone, M., Savla, K., Smith, S.L.: Dynamic vehicle
  routing for robotic systems.
\newblock Proceedings of the IEEE \textbf{99}(9), 1482--1504 (2011)

\bibitem{buyukkocak2021planning}
Buyukkocak, A.T., Aksaray, D., Yaz{\i}c{\i}o{\u{g}}lu, Y.: Planning of
  heterogeneous multi-agent systems under signal temporal logic specifications
  with integral predicates.
\newblock IEEE Robotics and Automation Letters \textbf{6}(2), 1375--1382 (2021)

\bibitem{claes2017decentralised}
Claes, D., Oliehoek, F., Baier, H., Tuyls, K., et~al.: Decentralised online
  planning for multi-robot warehouse commissioning.
\newblock In: International Conference on Autonomous Agents and Multiagent
  Systems (AAMAS), pp. 492--500 (2017)

\bibitem{cordeau2001unified}
Cordeau, J.F., Laporte, G., Mercier, A.: A unified tabu search heuristic for
  vehicle routing problems with time windows.
\newblock Journal of the Operational research society \textbf{52}(8), 928--936
  (2001)

\bibitem{dai2019task}
Dai, W., Lu, H., Xiao, J., Zheng, Z.: Task allocation without communication
  based on incomplete information game theory for multi-robot systems.
\newblock Journal of Intelligent \& Robotic Systems \textbf{94}(3-4), 841--856
  (2019)

\bibitem{durand2016complexity}
Durand, S., Gaujal, B.: Complexity and optimality of the best response
  algorithm in random potential games.
\newblock In: International Symposium on Algorithmic Game Theory, pp. 40--51.
  Springer (2016)

\bibitem{ellison1993learning}
Ellison, G.: Learning, local interaction, and coordination.
\newblock Econometrica: Journal of the Econometric Society pp. 1047--1071
  (1993)

\bibitem{even2003convergence}
Even-Dar, E., Kesselman, A., Mansour, Y.: Convergence time to nash equilibria.
\newblock In: International Colloquium on Automata, Languages, and Programming,
  pp. 502--513. Springer (2003)

\bibitem{gombolay2018fast}
Gombolay, M.C., Wilcox, R.J., Shah, J.A.: Fast scheduling of robot teams
  performing tasks with temporospatial constraints.
\newblock IEEE Transactions on Robotics \textbf{34}(1), 220--239 (2018)

\bibitem{guo2002distributed}
Guo, Y., Parker, L.E.: A distributed and optimal motion planning approach for
  multiple mobile robots.
\newblock In: IEEE International Conference on Robotics and Automation, vol.~3,
  pp. 2612--2619 (2002)

\bibitem{kapoutsis2017darp}
Kapoutsis, A.C., Chatzichristofis, S.A., Kosmatopoulos, E.B.: Darp: divide
  areas algorithm for optimal multi-robot coverage path planning.
\newblock Journal of Intelligent \& Robotic Systems \textbf{86}(3-4), 663--680
  (2017)

\bibitem{kreindler2013fast}
Kreindler, G.E., Young, H.P.: Fast convergence in evolutionary equilibrium
  selection.
\newblock Games and Economic Behavior \textbf{80}, 39--67 (2013)

\bibitem{kress2009temporal}
Kress-Gazit, H., Fainekos, G.E., Pappas, G.J.: Temporal-logic-based reactive
  mission and motion planning.
\newblock IEEE Transactions on Tobotics \textbf{25}(6), 1370--1381 (2009)

\bibitem{li2019multi}
Li, B., Moridian, B., Kamal, A., Patankar, S., Mahmoudian, N.: Multi-robot
  mission planning with static energy replenishment.
\newblock Journal of Intelligent \& Robotic Systems \textbf{95}(2), 745--759
  (2019)

\bibitem{Marden09}
Marden, J.R., Arslan, G., Shamma, J.S.: Cooperative control and potential
  games.
\newblock IEEE Transactions on Systems, Man, and Cybernetics, Part B:
  Cybernetics \textbf{39}(6), 1393--1407 (2009)

\bibitem{michael2008distributed}
Michael, N., Zavlanos, M.M., Kumar, V., Pappas, G.J.: Distributed multi-robot
  task assignment and formation control.
\newblock In: IEEE International Conference on Robotics and Automation, pp.
  128--133 (2008)

\bibitem{nunes2017taxonomy}
Nunes, E., Manner, M., Mitiche, H., Gini, M.: A taxonomy for task allocation
  problems with temporal and ordering constraints.
\newblock Robotics and Autonomous Systems \textbf{90}, 55--70 (2017)

\bibitem{peasgood2008complete}
Peasgood, M., Clark, C.M., McPhee, J.: A complete and scalable strategy for
  coordinating multiple robots within roadmaps.
\newblock IEEE Transactions on Robotics \textbf{24}(2), 283--292 (2008)

\bibitem{Peterson20}
Peterson, R., Buyukkocak, A.T., Aksaray, D., Yaz{\i}c{\i}o{\u{g}}lu, Y.:
  Decentralized safe reactive planning under {TWTL} specifications.
\newblock In: IEEE/RSJ International Conference on Intelligent Robots and
  Systems (2020)

\bibitem{preiss2017crazyswarm}
Preiss, J.A., Honig, W., Sukhatme, G.S., Ayanian, N.: Crazyswarm: A large
  nano-quadcopter swarm.
\newblock In: IEEE International Conference on Robotics and Automation (ICRA),
  pp. 3299--3304 (2017)

\bibitem{gonzalez2017fleets}
Gonzalez-de Santos, P., Ribeiro, A., Fernandez-Quintanilla, C., Lopez-Granados,
  F., Brandstoetter, M., Tomic, S., Pedrazzi, S., Peruzzi, A., Pajares, G.,
  Kaplanis, G., et~al.: Fleets of robots for environmentally-safe pest control
  in agriculture.
\newblock Precision Agriculture \textbf{18}(4), 574--614 (2017)

\bibitem{Seyedi19}
Seyedi, S., Yaz{\i}c{\i}o{\u{g}}lu, Y., Aksaray, D.: Persistent surveillance
  with energy-constrained uavs and mobile charging stations.
\newblock IFAC-PapersOnLine \textbf{52}(20), 193--198 (2019)

\bibitem{shah2010dynamics}
Shah, D., Shin, J.: Dynamics in congestion games.
\newblock ACM SIGMETRICS Performance Evaluation Review \textbf{38}(1), 107--118
  (2010)

\bibitem{Thakur13}
Thakur, D., Likhachev, M., Keller, J., Kumar, V., Dobrokhodov, V., Jones, K.,
  Wurz, J., Kaminer, I.: Planning for opportunistic surveillance with multiple
  robots.
\newblock In: IEEE/RSJ International Conference on Intelligent Robots and
  Systems, pp. 5750--5757 (2013)

\bibitem{Tumer04}
Tumer, K., Wolpert, D.H.: Collectives and the design of complex systems.
\newblock Springer Science \& Business Media (2004)

\bibitem{ulusoy2013optimality}
Ulusoy, A., Smith, S.L., Ding, X.C., Belta, C., Rus, D.: Optimality and
  robustness in multi-robot path planning with temporal logic constraints.
\newblock The International Journal of Robotics Research \textbf{32}(8),
  889--911 (2013)

\bibitem{wang2020coupled}
Wang, H., Chen, W., Wang, J.: Coupled task scheduling for heterogeneous
  multi-robot system of two robot types performing complex-schedule order
  fulfillment tasks.
\newblock Robotics and Autonomous Systems p. 103560 (2020)

\bibitem{Yazicioglu13NECSYS}
Yaz{\i}c{\i}o{\u{g}}lu, A.Y., Egerstedt, M., Shamma, J.S.: A game theoretic
  approach to distributed coverage of graphs by heterogeneous mobile agents.
\newblock IFAC Proc. Volumes \textbf{46}(27), 309--315 (2013)

\bibitem{Yazicioglu17TCNS}
Yaz{\i}c{\i}o{\u{g}}lu, A.Y., Egerstedt, M., Shamma, J.S.: Communication-free
  distributed coverage for networked systems.
\newblock IEEE Transactions on Control of Network Systems \textbf{4}(3),
  499--510 (2017)

\bibitem{Young04}
Young, H.P.: Strategic learning and its limits.
\newblock Oxford university press (2004)

\bibitem{yu2016optimal}
Yu, J., LaValle, S.M.: Optimal multirobot path planning on graphs: Complete
  algorithms and effective heuristics.
\newblock IEEE Transactions on Robotics \textbf{32}(5), 1163--1177 (2016)

\bibitem{Zhu13}
Zhu, M., Mart\'{\i}nez, S.: Distributed coverage games for energy-aware mobile
  sensor networks.
\newblock SIAM Journal on Control and Optimization \textbf{51}(1), 1--27 (2013)

\end{thebibliography}

\end{document}